%% file: main.tex
\documentclass{article}

\usepackage[T1]{fontenc}
\usepackage[htt]{hyphenat}

\usepackage{microtype}
\usepackage{graphicx}
\usepackage{subfigure}
\usepackage{booktabs} 
\usepackage{natbib}
\usepackage{algorithm}
\usepackage{algorithmic}
\input{preamble}

\newcommand\blfootnote[1]{%
  \begingroup
  \renewcommand\thefootnote{}\footnote{#1}%
  \addtocounter{footnote}{-1}%
  \endgroup
}

\usepackage{breakcites}
\usepackage{authblk}
\usepackage[margin=1.35in]{geometry}
\usepackage{enumitem}

\usepackage{hyperref}

\hypersetup{
     colorlinks = true,
     linkcolor = brown,
     anchorcolor = blue,
     citecolor = teal,
     filecolor = blue,
     urlcolor = black
     }

\title{Separation Results between Fixed-Kernel and Feature-Learning Probability Metrics}

\author[a]{Carles Domingo-Enrich}
\author[b]{Youssef Mroueh}
\affil[a]{Courant Institute of Mathematical Sciences, New York University}
\affil[b]{IBM Research AI}

\begin{document}

\maketitle

\begin{abstract}
Several works in implicit and explicit generative modeling empirically observed that feature-learning discriminators  outperform  fixed-kernel discriminators in terms of the sample quality
of the models. We provide separation results between probability metrics with fixed-kernel and feature-learning discriminators using the function classes $\mathcal{F}_2$  and $\mathcal{F}_1$ respectively, which were developed to study overparametrized two-layer neural networks.
In particular, we construct pairs of distributions over hyper-spheres that can not be discriminated by  fixed kernel $(\mathcal{F}_2)$ integral probability metric (IPM) and Stein discrepancy (SD) in high dimensions, but that can be discriminated by their feature learning ($\mathcal{F}_1$) counterparts.   
To further study the separation we provide links between the $\mathcal{F}_1$ and $\mathcal{F}_2$ IPMs with sliced Wasserstein distances. Our work suggests that fixed-kernel discriminators perform worse than their feature learning counterparts because their corresponding metrics are weaker.
\end{abstract}

\section{Introduction} \label{sec:intro}
The field of generative modeling, whose aim is to generate artificial samples of some target distribution given some true samples of it, is broadly divided into two types of models \citep{mohamed2017learning}: explicit generative models, which involve learning an estimate of the log-density of the target distribution which is then sampled (e.g. energy-based models), and implicit generative models, where samples are generated directly by transforming some latent variable (e.g. generative adversarial networks \citep{GANoriginal}, normalizing flows \citep{rezende2015variational}). \blfootnote{Correspondence to: \textsuperscript{a}\url{cd2754@nyu.edu}, \textsuperscript{b}\url{mroueh@us.ibm.com}}

Several works have observed experimentally that both in implicit and in explicit generative models, using `adaptive' or `feature-learning' function classes as discriminators yields better generative performance than `lazy' or `kernel' function classes. Within implicit models, \cite{li2017mmd} show that generative moment matching networks (GMMN) generate significantly better samples for the CIFAR-10 and MNIST datasets when using maximum mean discrepancy (MMD) with learned instead of fixed features. For a related method and in a similar spirit, \cite{santos2019learning2} show that for image generation, fixed-feature discriminators are only successful when we take an amount of features exponential in the intrinsic dimension of the dataset. \cite{genevay2018learning} study implicit generative models with the Sinkhorn divergence as discriminator, 
and they also show that other than for simple datasets like MNIST, learning the Wasserstein cost is crucial for good performance.

As to explicit models, \cite{grathwohl2020learning} train energy-based models with a Stein discrepancy based on neural networks and show improved performance with respect to kernel classes. 
\cite{chang2020kernel} show that Stein variational gradient descent (SVGD) fails in high dimensions, and that learning the kernel helps. 
Given the abundant experimental evidence, the aim of this work is to provide some theoretical results that showcase the advantages of feature-learning over kernel discriminators. For the sake of simplicity, we compare the discriminative behavior of two function classes $\mathcal F_1$ and~$\mathcal F_2$, arising from infinite-width two-layer neural networks with different norms penalties on its weights \citep{bach2017breaking}. $\mathcal F_1$ displays an adaptive behavior, while $\mathcal F_2$ is an RKHS which consequently has a lazy behavior. Namely, our main contributions are:
\begin{enumerate}[label=(\roman*),leftmargin=0.5cm]
    \item We construct a sequence of pairs of distributions over hyperspheres $\mathbb{S}^{d-1}$ of increasing dimensions, such that the $\mathcal F_2$ integral probability metric (IPM) between the pair decreases exponentially in the dimension, while the $\mathcal F_1$ IPM remains high.
    \item We construct a sequence of pairs of distributions over $\mathbb{S}^{d-1}$ such that the $\mathcal F_2$ Stein discrepancy (SD) between the pair decreases exponentially in the dimension, while the $\mathcal F_1$ SD remains high.
    \item We prove polynomial upper and lower bounds between the $\mathcal F_1$ IPM and the max-sliced Wasserstein distance for distributions over Euclidean balls. For a class $\tilde{\mathcal{F}}_2$ related to $\mathcal{F}_2$, we prove similar upper and lower bounds between the $\tilde{\mathcal{F}}_2$ IPM and the sliced Wasserstein distance for distributions over Euclidean balls. 
\end{enumerate}
Our findings reinforce the idea that generative models with kernel discriminators have worse performance because their corresponding metrics are weaker and thus unable to distinguish between different distributions, especially in high dimensions.


\section{Related work}
A recent line of research has studied the question of how neural networks compare to kernel methods, with a focus on supervised learning problems.
\citet{bach2017breaking} 
shows the approximation benefits of the~$\mathcal F_1$ space for adapting to low-dimensional structures compared to the (kernel) space~$\mathcal F_2$; an analysis that we leverage.
The function space~$\mathcal F_1$ was also studied by~\citet{ongie2019function,savarese2019infinite,williams2019gradient}, which focus on the ReLU activation function. 
More recently, 
several works showed that wide neural networks trained with gradient methods may behave like kernel methods in certain regimes~\citep[see, e.g.,][]{jacot2018neural}.
Examples of works that compare `active/feature-learning' and `kernel/lazy' regimes for supervised learning include~\cite{chizat2020implicit,ghorbani2019limitations,wei2020regularization,woodworth2020kernel}, and~\cite{domingoenrich2021energybased} for energy-based models.
We are not aware of any works that study how feature-learning function classes and kernel classes differ as discriminators for IPMs or Stein discrepancies. 

It turns out that the $\mathcal{F}_2$ integral probability metric that we study is in fact 
MMD for certain kernels that often admit a closed form \citep{leroux2007continuous, cho2009kernel,bach2017breaking}. MMDs are probability metrics that were first introduced by \cite{gretton2007kernel, gretton2012akernel} for kernel two-sample tests, and that have enjoyed ample success with the advent of deep-learning-based generative modeling as discriminating metrics: \cite{li2015generative} and \cite{dziugaite2015training} introduced  
GMMN,
which differ from GANs in that the discriminator network is replaced by a fixed-kernel MMD. 
\cite{li2017mmd} introduces an improvement on GMMN by using the MMD loss on learned features. From this viewpoint, our separation results in \autoref{sec:sep_f1_f2_ipm} can be interpreted as instances in which a given fixed-kernel MMD provably has less discriminative power than adaptive discriminators.

Other related work includes the Stein discrepancy literature. Stein's method \citep{stein1972abound} dates to the 1970s.
\citet{gorham2015measuring} introduced a computational approach to compute the Stein discrepancy in order to assess sample quality. Later, \citet{chwialkowski2016gretton}, \citet{liu2016akernelized} and \citet{gorham2017measuring} introduced the more practical kernelized Stein discrepancy (KSD) for goodness-of-fit tests. \citet{liu2016stein} introduced SVGD, the first method to use the KSD to obtain samples from a distribution. \citet{barp2019minimum} employed KSD to train parametric generative models, and \citet{grathwohl2020learning} trained models replacing KSD by a neural-network-based SD. 

Our work also touches on sliced and spiked Wasserstein distances. Sliced Wasserstein distances were introduced first by \citet{kolouri2016sliced, kolouri2019generalized}. Spiked Wasserstein distances, which are a generalization, were studied later by \cite{paty2019subspace}, and they also appear in \cite{nilesweed2019estimation} as a good statistical estimator. \cite{nadjahi2020statistical} and \cite{lin2021projection} have studied statistical properties of sliced and spiked Wasserstein distances, respectively.

\section{Framework}

\subsection{Notation}
If $V$ is a normed vector space, we use $\mathcal{B}_V(\beta)$ to denote the closed ball of $V$ of radius $\beta$, and $\mathcal{B}_V := \mathcal{B}_V(1)$ for the unit ball. If $K$ denotes a subset of the Euclidean space, $\mathcal{P}(K)$ is the set of Borel probability measures, $\mathcal{M}(K)$ is the space of finite signed Radon measures and $\mathcal{M}^{+}(K)$ is the set of finite positive Radon measures. If $\gamma$ is a signed Radon measure over $K$, then ${\|\gamma\|}_{\text{TV}}$ is the total variation (TV) norm of $\gamma$. Throughout the paper, and unless otherwise specified, $\sigma : \R \to \R$ denotes a generic non-linear activation function. We use $(\cdot)_{+} : \R \rightarrow \R$ to denote the ReLu activation, defined as $(x)_{+} = \max \{x, 0\}$. $\tau$ denotes the uniform probability measure over a space that depends on the context. We use $\mathbb{S}^d$ for the $d$-dimensional hypersphere and $\log$ for the natural logarithm.

\subsection{Overparametrized two-layer neural network spaces} \label{subsec:framework_overparametrized}
\paragraph{Feature-learning regime.} We define $\mathcal{F}_1$ as the Banach space of functions $f : K \rightarrow \R$ such that for some $\gamma \in \mathcal{M}(\mathbb{S}^d)$, for all $x \in K$ we have $f(x) = \int_{\mathbb{S}^d} \sigma(\langle \theta, x \rangle) \ d\gamma(\theta)$ for some signed Radon measure $\gamma$ \citep{bach2017breaking}. The norm of $\mathcal{F}_1$ is defined as
    $\|f\|_{\mathcal{F}_1} = \inf \left\{ {\|\gamma\|}_{\text{TV}} \ | f(\cdot) = \int_{\mathbb{S}^d} \sigma(\langle \theta, \cdot \rangle) \ d\gamma(\theta) \right\}.$
    
\paragraph{Kernel regime.} We define $\mathcal{F}_2$ as the (reproducing kernel) Hilbert space of functions $f : K \rightarrow \R$ such that for some absolutely continuous $\rho \in \mathcal{M}(\mathbb{S}^d)$ with $\frac{d\rho}{d\tau} \in \mathcal{L}^2(\mathbb{S}^d)$ (where $\tau$ is the uniform probability measure over $\mathbb{S}^d$), we have that for all $x \in K$,  $f(x) = \int_{\mathbb{S}^d} \sigma(\langle \theta, x \rangle) \ d\rho(\theta)$. The norm of $\mathcal{F}_2$ is defined as $\|f\|_{\mathcal{F}_2}^2 = \inf \left\{ \int_{\mathbb{S}^d} h(\theta)^2 \ d\tau(\theta) \ | \ f(\cdot) = \int_{\mathbb{S}^d} \sigma(\langle \theta, \cdot \rangle) \ h(\theta) \ d\tau(\theta) \right\}$. As an RKHS, the kernel of $\mathcal{F}_2$ is 
\begin{align} \label{eq:f2_kernel}
k(x,y) = \int_{\mathbb{S}^d} \sigma(\langle x, \theta \rangle) \sigma(\langle y, \theta \rangle) d\tau(\theta).
\end{align}
Such kernels admit closed form expressions for different choices of activation functions, among which ReLu \citep{leroux2007continuous, cho2009kernel,bach2017breaking}.

Remark that since $\int |h(\theta)| d\tau(\theta) \leq (\int h(\theta)^2 \ d\tau(\theta))^{1/2}$ by the Cauchy-Schwarz inequality, we have $\mathcal{F}_2 \subset \mathcal{F}_1$. 
In particular, when $\sigma$ is the ReLu unit, \cite{bach2017breaking} shows that two-layer networks with a single neuron belong to $\mathcal{F}_1$ but not to $\mathcal{F}_2$, and their $L^2$ approximations in $\mathcal{F}_2$ have exponentially high norm in the dimension. Informally, one should understand $\mathcal{F}_1$ as the space of two-layer networks where both the input layer and output layer parameters are trained, in the limit of an infinite number of neurons. On the other hand, $\mathcal{F}_2$ is the space of infinite-width two-layer networks where only the output layer parameters are trained while the input layer parameters are sampled uniformly on the sphere and kept fixed.

\section{$\mathcal{F}_1$ and $\mathcal{F}_2$ Integral Probability Metrics}
Let $K$ be a subset of $\R^{d+1}$. Integral probability metrics (IPM) are pseudometrics on $\mathcal{P}(K)$ of the form
\begin{align} \label{eq:def_F_ipm}
d_{\mathcal{F}}(\mu,\nu) = \sup_{f\in \mathcal{F}}  \mathbb{E}_{x\sim \mu} f(x) - \mathbb{E}_{x\sim \nu} f(x),
\end{align}
where $\mathcal{F}$ is a class of functions from $K$ to $\R$.

\paragraph{$\mathcal{F}_2$ IPM or $\mathcal{F}_2$ MMD.} One possible choice for $\mathcal{F}$ is the unit ball $\mathcal{B}_{\mathcal{F}_2}$ of $\mathcal{F}_2$. Since $\mathcal{F}_2$ is an RKHS with kernel $k$, the corresponding IPM is in fact a maximum mean discrepancy (MMD) \citep{gretton2007kernel} and it can be shown (\autoref{lem:f1_exp} in \autoref{sec:app_f1_f2}) to take the form 
\begin{align} \label{eq:def_f2_ipm}
d^2_{\mathcal{B}_{\mathcal{F}_2}}(\mu,\nu) = \int_{\mathbb{S}^d} \left( \int_K \sigma(\langle x, \theta \rangle) d(\mu-\nu)(x) \right)^2 d\tau(\theta).
\end{align}
Notice that for any feature $\theta \in \mathbb{S}^d$, $\mathbb{E}_{x \sim p} \sigma(\langle x, \theta \rangle)$ can be seen as a generalized moment of $p$. 
$d_{\mathcal{B}_{\mathcal{F}_2}}$ can be seen as the $L^2$ distance between generalized moments of $\mu$ and $\nu$ as functions of $\theta \in \mathbb{S}^d$.

\textbf{$\mathcal{F}_1$ IPM.} An alternative choice for $\mathcal{F}$ is the unit ball $\mathcal{B}_{\mathcal{F}_1}$ of $\mathcal{F}_1$. The IPM for the unit ball of $\mathcal{F}_1$ can be developed (\autoref{lem:f2_exp} in \autoref{sec:app_f1_f2}) into 
\begin{align} \label{eq:def_f1_ipm}
d_{\mathcal{B}_{\mathcal{F}_1}}(\mu,\nu) = \sup_{\theta \in \mathbb{S}^d} \left| \int_K \sigma(\langle x, \theta \rangle) d(\mu-\nu)(x) \right|.
\end{align}
\vskip -0.2in
Observe that $d_{\mathcal{B}_{\mathcal{F}_1}}$ is the $L^\infty$ distance between generalized moments of $\mu$ and $\mu$ as functions of $\theta \in \mathbb{S}^d$. That is, instead of averaging over features, all the weight is allocated to the feature at which the generalized moment difference is larger.

We will provide separate results for two interesting choices for $K$: (i) for $K = \mathbb{S}^d$, we obtain neural network discriminators without bias term which are amenable to analysis using the theory of spherical harmonics; and (ii) for $K = \R^d \times \{1\}$, we obtain neural networks discriminators with a bias~term which is encoded by the last component (notice that probability measures over $\R^d$ can be mapped trivially to probability measures over $\R^d \times \{1\}$). We will write $\mathcal{F}_1(K)$ or $\mathcal{F}_2(K)$ for specific $K$ when it is not clear by the context.

A function $f : \R \rightarrow \R$ is $\alpha$-positive homogeneous function if for all $r \geq 0, x \in \R$, $f (r x) = r^\alpha f(x)$. One-dimensional $\alpha$-positive homogeneous functions can be written in a general form as \begin{align} \label{eq:alpha_positive}
    f(x) = a (x)_{+}^{\alpha} + b (-x)_{+}^{\alpha}.
\end{align}
where $a, b \in \R$ are arbitrary. When the activation function $\sigma$ is $\alpha$-positive homogeneous, \autoref{thm:distances} shows that the $\mathcal{F}_1$ and $\mathcal{F}_2$ IPMs are distances when $K = \R^d \times \{1\}$ if $a, b$ fulfill a certain condition which is satisfied by the ReLu activation, but they are \textit{not} distances when $K = \mathbb{S}^d$. See \autoref{thm:ipmdistbias} and \autoref{thm:ipmdistsphere} in \autoref{sec:distances} for the proof.

\begin{restatable}{thm}{thmdistances} \label{thm:distances}
For any non-negative integer $\alpha$, let $\sigma: \R \rightarrow \R$ be an $\alpha$-positive homogeneous activation function of the form \eqref{eq:alpha_positive}. If $(-1)^{\alpha} a - b \neq 0$ and $K = \R^d \times \{1\}$, both the $\mathcal{F}_1$ and $\mathcal{F}_2$ IPMs are distances on $\mathcal{P}(K)$. If $K = \mathbb{S}^d$, both the $\mathcal{F}_1$ and $\mathcal{F}_2$ IPMs are \textit{not} distances on $\mathcal{P}(K)$, as there exist pairs of different measures for which the IPMs evaluate to zero. 
\end{restatable}

In other words, \autoref{thm:distances} states that certain fixed-kernel and feature-learning infinite neural networks  with RELU or leaky RELU non-linearity, yield distances when we include a bias term, but not when the inputs lie in a hypersphere. This result sheds light on when the ``neural net distance'' introduced by \cite{arora2017generalization} is indeed a distance.

\section{Separation between the $\mathcal{F}_1$ and $\mathcal{F}_2$ IPMs} \label{sec:sep_f1_f2_ipm}
In this section for each dimension $d \geq 2$, we construct a pair of probability measures $\mu_d,\nu_d$ over $\mathcal{P}(\mathbb{S}^{d-1})$ such that the $\mathcal{F}_1$ IPM between $\mu_d$ and $\nu_d$ stays constant along the dimension, while the $\mathcal{F}_2$ IPM decreases exponentially.

\textbf{Legendre harmonics and Legendre polynomials.} Let $e_d \in \R^d$ be the $d$-th vector of the canonical basis. There is a unique homogeneous harmonic polynomial $L_{k,d}$ of degree $k$ over $\R^d$ such that: (i) $L_{k,d}(Ax) = L_{k,d}(x)$ for all orthogonal matrices that leave $e_d$ invariant, and (ii) $L_{k,d}(e_d) = 1$. This polynomial receives the name of \textit{Legendre harmonic}, and its restriction to $\mathbb{S}^{d-1}$ is indeed a spherical harmonic of order $k$.
If we express an arbitrary $\xi_{(d)} \in \mathbb{S}^{d-1}$ as $\xi_{(d)} = t e_d + (1-t^2)^{1/2} \xi_{(d-1)}$, where $\xi_{(d-1)} \perp e_d$, we can define the \textit{Legendre polynomial} of degree $k$ in dimension $d$ as $P_{k,d}(t) := L_{k,d}(\xi_{(d)})$ by the invariance of $L_{k,d}$ (it is not straightforward that $P_{k,d}(t)$ is a polynomial on $t$, see Sec.~2.1.2 of \cite{atkinson2012spherical}). Conversely, $L_{k,d}(x) = P_{k,d}(\langle e_d, x \rangle)$ for any $x \in \mathbb{S}^{d-1}$, and by homogeneity, $L_{k,d}(x) = \|x\|^k P_{k,d}(\langle e_d, x \rangle/\|x\|)$ for any $x \in \R^d$. Legendre polynomials can also be characterized as the orthogonal sequence of polynomials on $[-1,1]$ such that $P_{k,d}(1) = 1$ and $\int_{-1}^1 P_{k,d}(t) P_{l,d}(t) (1-t^2)^{\frac{d-3}{2}} dt = 0,$ for $k \neq l$.


\paragraph{The pair $\mu_d$ and $\nu_d$.} We define $\mu_d$ and $\nu_d$ as the probability measures over $\mathbb{S}^{d-1}$ with densities
\begin{align}
\begin{split} \label{eq:densities_mu_nu}
    \frac{d\mu_d}{d\lambda} &= 
    \begin{cases}
        \frac{\gamma_{k,d} L_{k,d}(x)}{|\mathbb{S}^{d-1}|} &\text{if } L_{k,d}(x) > 0  \\
        0 &\text{if } L_{k,d}(x) \leq 0
    \end{cases}, \\
    \frac{d\nu_d}{d\lambda} &= 
    \begin{cases}
        0 &\text{if } L_{k,d}(x) > 0  \\
        \frac{-\gamma_{k,d} L_{k,d}(x)}{|\mathbb{S}^{d-1}|} &\text{if } L_{k,d}(x) \leq 0
    \end{cases}.
\end{split}    
\end{align}
for some $k \geq 2$ and some
$\gamma_{k,d} \geq 0$, where $\lambda$ is the Hausdorff measure over $\mathbb{S}^{d-1}$. Namely,
\begin{restatable}{prop}{propqdprob} \label{prop:qd_prob}
If we choose $\gamma_{k,d} = 2 \left(\int_{\mathbb{S}^{d-1}} |L_{k,d}(x)| \ d\tau(x) \right)^{-1}$, then $\mu_d$ and $\nu_d$ are probability measures.
\end{restatable}

\begin{figure}
    \centering
    \includegraphics[width=0.49\textwidth]{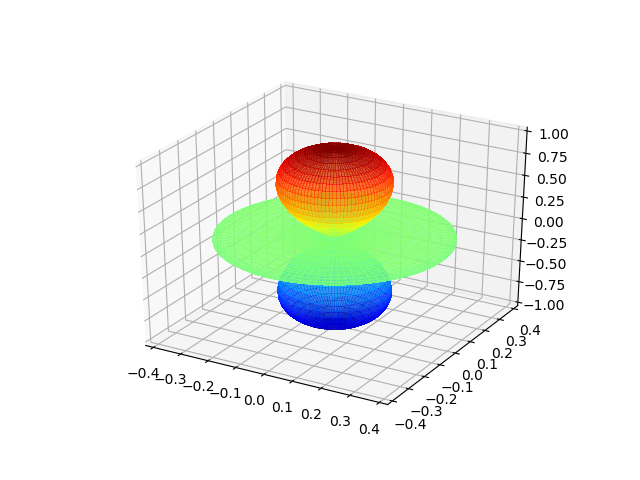}%
    \includegraphics[width=0.49\textwidth]{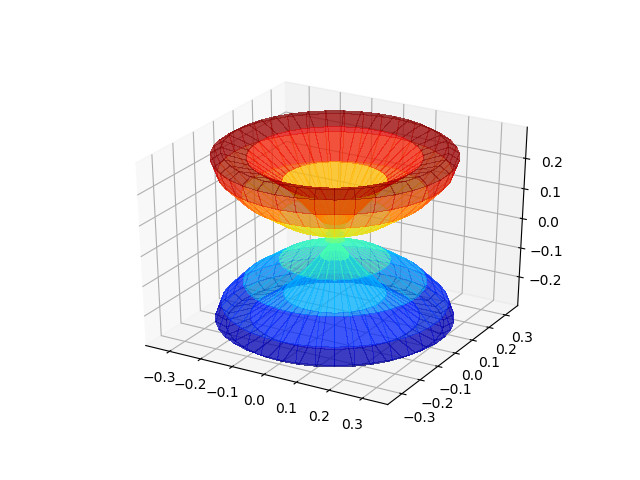}
    \vspace{-5pt}
    \caption{3D polar plot representing the densities of the measures $\mu_d$ (left) and $\nu_d$ (right), for the choices $d=3$, $k=4$. In each direction, the distance from the origin to the surface is proportional to the density of the measure.}
    \label{fig:mu_d_nu_d_polar_plot}
\end{figure}

\autoref{fig:mu_d_nu_d_polar_plot} shows a representation of the measures $\mu_d, \nu_d$ for $d=3,k=4$, where one can see that they allocate mass in different regions of the sphere. We are now ready to state our separation result, which is proved in \autoref{sec:proof_sep}.
\begin{restatable}{thm}{thmsepipm} \label{thm:separation_ipm}
Let $\sigma : \R \to \R$ be an activation function that is bounded in $[-1,1]$. 
For any $d\geq 2$ and $k \geq 1$, if we set $\gamma_{k,d}$ as in \autoref{prop:qd_prob} we have that
\begin{align} \label{eq:d_f1_separation}
d_{\mathcal{B}_{\mathcal{F}_1}}(\mu_d,\nu_d) = 
\frac{2 \left|\int_{-1}^{1} P_{k,d}(t) \sigma(t) (1-t^2)^{\frac{d-3}{2}} \ dt \right|}{\int_{-1}^{1} |P_{k,d}(t)| (1-t^2)^{\frac{d-3}{2}} dt},
\end{align}
and
\begin{align} \label{eq:distances_ratio0}
    \frac{d_{\mathcal{B}_{\mathcal{F}_1}}(\mu_d,\nu_d)}{d_{\mathcal{B}_{\mathcal{F}_2}}(\mu_d,\nu_d)} = \sqrt{N_{k,d}} = \sqrt{\frac{(2k + d - 2)(k + d -3)!}{k! (d-2)!}},
\end{align}
where $N_{k,d}$ is the dimension of the space of spherical harmonics of order $k$ over $\mathbb{S}^{d-1}$. That is, 
\begin{align} \label{eq:log_ratio}
    \log \left( \frac{d_{\mathcal{B}_{\mathcal{F}_1}}(\mu_d,\nu_d)}{d_{\mathcal{B}_{\mathcal{F}_2}}(\mu_d,\nu_d)} \right) 
    &= \frac{1}{2} \left( k \log\left(\frac{k+d-3}{k} \right) + (d-2) \log \left(\frac{k+d-3}{d-2} \right) \right) + O(\log(k+d)).
\end{align}
\end{restatable}
From \eqref{eq:log_ratio} we see that choosing the parameter $k$ of the same order as $d$, $d_{\mathcal{B}_{\mathcal{F}_1}}(\mu_d,\nu_d)$ is exponentially larger than $d_{\mathcal{B}_{\mathcal{F}_2}}(\mu_d,\nu_d)$ in the dimension $d$. Equation \eqref{eq:d_f1_separation} holds regardless of the choice of the activation function $\sigma$, and decreases very slowly in $d$ for the ReLu activation, as shown in \autoref{fig:f1_f2_ipm_separation_figure}. This result suggests that in high dimensions there exist high frequency densities that can be distinguished by feature-learning IPM discriminators but not by their fixed-kernel counterpart, and that may explain the differences in generative modeling performance for GMMN and Sinkhorn divergence (\autoref{sec:intro}).

The key idea for the proof of \autoref{thm:separation_ipm} is that the Legendre harmonics 
$L_{k,d}$ have constant $L^{\infty}$ norm equal to 1 (see equation \eqref{eq:uniform_bound} in \autoref{sec:preliminaries}), but their $L^2$ norm decreases as $1/N_{k,d}$ (see equation \eqref{eq:norm_prob} in \autoref{sec:proof_sep}). The proof boils down to relating $d_{\mathcal{B}_{\mathcal{F}_1}}(\mu_d,\nu_d)$ to the $L^{\infty}$ norm of $L_{k,d}$, and $d_{\mathcal{B}_{\mathcal{F}_2}}(\mu_d,\nu_d)$ to its $L^{2}$ norm.

\section{Separation between $\mathcal{F}_1$ and $\mathcal{F}_2$ Stein discrepancies} \label{sec:sep_sd}
The arguments to derive the separation result in \autoref{sec:sep_f1_f2_ipm} can be leveraged to obtain a similar separation for the Stein discrepancy, which helps explain why for Stein discrepancy energy-based models (EBMs) and SVGD feature learning yields improved performance.  

\subsection{Stein operator and Stein discrepancy}
As shown by \cite{domingoenrich2021energybased}, for a probability measure $\nu$ on the sphere $\mathbb{S}^{d-1}$ with a continuous and almost everywhere differentiable density $\frac{d\nu}{d\tau}$, the \textit{Stein operator} $\mathcal{A}_{\nu} : \mathbb{S}^{d-1} \rightarrow \R^{d\times d}$ is defined as
\begin{align} \label{eq:stein_operator}
    (\mathcal{A}_{\nu} h)(x) = \bigg(\nabla \log \left(\frac{d\nu}{d\tau}(x) \right) - (d-1) x \bigg) h(x)^{\top} + \nabla h(x),
\end{align}
for any $h : \mathbb{S}^{d-1} \rightarrow \mathbb{R}^{d}$ that is continuous and almost everywhere differentiable, where $\nabla$ denotes the Riemannian gradient.
That is, for any $h : \mathbb{S}^{d-1} \rightarrow \mathbb{R}^{d}$ that is continuous and almost everywhere differentiable, the Stein identity holds: $\mathbb{E}_{\nu} [(\mathcal{A}_{\nu} h)(x)] = 0$.

If $\mathcal{H}$ is a class of functions from $\mathbb{S}^{d-1}$ to $\R^{d}$, the Stein discrepancy \citep{gorham2015measuring, liu2016akernelized} for $\mathcal{H}$ is a non-symmetric functional defined on pairs of probability measures over $K$ as
\begin{align} \label{eq:sd_def}
    \text{SD}_{\mathcal{H}}(\nu_1, \nu_2) = \sup_{h \in \mathcal{H}} \mathbb{E}_{\nu_1} [\text{Tr}(\mathcal{A}_{\nu_2} h(x))].
\end{align}
When $\mathcal{H} = \mathcal{B}_{\mathcal{H}_0^{d}} = \{ (h_i)_{i=1}^{d} \in \mathcal{H}_0^{d} \ | \ \sum_{i=1}^{d} \|h_i\|_{\mathcal{H}_0}^2 \leq 1 \}$ 
for some reproducing kernel Hilbert space (RKHS) $\mathcal{H}_0$ with kernel $k$ with continuous second order partial derivatives, there exists a closed form for the problem \eqref{eq:sd_def} 
and the corresponding object is known as kernelized Stein discrepancy (KSD) \cite{liu2016akernelized, gorham2017measuring}.
When the domain is $\mathbb{S}^{d-1}$, the KSD takes the following form (Lemma 5, \cite{domingoenrich2021energybased}):
\begin{align} \label{eq:KSD_def}
    \text{KSD}(\nu_1, \nu_2) = \text{SD}^2_{\mathcal{B}_{\mathcal{H}_0^{d}}}(\nu_1, \nu_2) = \mathbb{E}_{x,x' \sim \nu_1} [u_{\nu_2}(x,x')], 
\end{align}
where we have $u_{\nu}(x,x') = (s_{\nu}(x) - (d-1) x)^\top(s_{\nu}(x') - (d-1) x') k(x,x') + (s_{\nu}(x) - (d-1) x)^\top \nabla_{x'} k(x,x') + (s_{\nu}(x') - (d-1) x')^\top \nabla_{x} k(x,x') + \text{Tr}(\nabla_{x,x'} k(x,x'))$, and we use $\tilde{u}_{\nu}(x,x')$ to denote the sum of the first three terms (remark that the fourth term does not depend on $\nu$). Here we have used the notation $s_{\nu}(x) = \nabla \log (\frac{d\nu}{d\tau}(x))$, which is known as the score function.

\subsection{Separation result} \label{sec:separation_sd}
We show a separation result between the two cases:
\begin{itemize}[leftmargin=*]
    \item $\mathcal{F}_1$ Stein discrepancy: $\mathcal{H} = \mathcal{B}_{\mathcal{F}_1^d} = \{ (h_i)_{i=1}^{d} \in \mathcal{F}_1^{d} \ | \ \sum_{i=1}^{d} \|h_i\|_{\mathcal{F}_1}^2 \leq 1 \}$. This discriminator set initially appeared as a particular configuration in the framework of \cite{huggins2018random}, and its statistical properties for energy based model training were later studied by \cite{domingoenrich2021energybased}. 
    \item $\mathcal{F}_2$ Stein discrepancy: $\mathcal{H} = \mathcal{B}_{\mathcal{F}_2^d} = \{ (h_i)_{i=1}^{d} \in \mathcal{F}_2^{d} \ | \ \sum_{i=1}^{d} \|h_i\|_{\mathcal{F}_2}^2 \leq 1 \}$. Since $\mathcal{F}_2$ is an RKHS, this corresponds to a KSD with the kernel $k$. However, particular care must be taken in checking that the kernel $k$ has continuous second order partial derivatives, which might not always be the case (i.e. with $\alpha = 1$).  
\end{itemize}

\paragraph{The pair $\mu_d$ and $\nu_d$.} For $d \geq 2$, we set $\mu_d$ to be the uniform Borel probability measure over $\mathbb{S}^{d-1}$.
We define $\nu_d$ as the probability measure over $\mathbb{S}^{d-1}$ with density
\begin{align} \label{eq:nu_density}
\frac{d\nu_d}{d\lambda}(x) = \frac{\exp \left( \gamma_{k,d} L_{k,d}(x) \right)}{\int_{\mathbb{S}^{d-1}} \exp \left(\gamma_{k,d} L_{k,d}(x) \right) d\lambda(x)}
\end{align}
for some $\gamma_{k,d} \in \R$ that we will specify later on and some $k \geq 2$.

\begin{restatable}{thm}{thmseparationsd} \label{thm:sep_sd}
Let $\sigma: \R \rightarrow \R$ be an $\alpha$-positive homogeneous activation function of the form \eqref{eq:alpha_positive} such that $a + (-1)^{k+1} b \neq 0$. For all $k \geq 1, \ d \geq 2$ we can choose $\gamma_{k,d} \in [-1,1]$ such that $\text{SD}_{\mathcal{B}_{\mathcal{F}_1^d}}(\mu_d,\nu_d) = 1$ and
\begin{align} \label{eq:sd_sep}
    \frac{\text{SD}_{\mathcal{B}_{\mathcal{F}_1^d}}(\mu_d,\nu_d)}{\text{SD}_{\mathcal{B}_{\mathcal{F}_2^d}}(\mu_d,\nu_d)} \geq 
    \frac{\frac{k(d+k-3)}{\alpha + 1}}{\sqrt{\frac{2}{N_{k,d}} \left(k (k + d -2) \left( \frac{d+\alpha-2}{\alpha+1} \right)^2 + \left( \frac{k(d+k-3)}{\alpha + 1} \right)^2 \right) }}
\end{align}
That is,
\begin{align}
\begin{split} \label{eq:sd_sep_log}
    \log \left( \frac{\text{SD}_{\mathcal{B}_{\mathcal{F}_1^d}}(\mu_d,\nu_d)}{\text{SD}_{\mathcal{B}_{\mathcal{F}_2^d}}(\mu_d,\nu_d)} \right) \geq &\frac{1}{2} \left( k \log\left(\frac{k+d-3}{k} \right) + (d-2) \log \left(\frac{k+d-3}{d-2} \right) \right) + O(\log(k+d))
\end{split}
\end{align}
\end{restatable}
As in \autoref{thm:separation_ipm}, from \eqref{eq:sd_sep_log} we see that choosing the parameter $k$ of the same order as $d$, $\text{SD}_{\mathcal{B}_{\mathcal{F}_1^d}}(\mu_d,\nu_d)$ is exponentially larger than $\text{SD}_{\mathcal{B}_{\mathcal{F}_2^d}}(\mu_d,\nu_d)$ in the dimension $d$. This result suggests that in high dimensions there exist high frequency densities that can be distinguished by feature-learning Stein Discrepancy discriminators but not by their fixed-kernel counterpart, and that may explain the differences in generative modeling performance for Stein discrepancy EBMs and SVGD (\autoref{sec:intro}).

\section{Bounds of $\mathcal{F}_1$ and $\mathcal{F}_2$ IPMs by sliced Wasserstein distances} \label{sec:bounds_f1_f2}

$\mathcal{F}_1$ and $\mathcal{F}_2$ IPMs measure differences of densities by slicing the input space and then maximizing (resp. averaging) the appropriate quantities. Max-sliced and sliced Wasserstein distances work, which have been studied by several works, work in an analogous fashion; one projects the distributions onto one-dimensional subspaces, and then maximizes or averages over the subspaces. Unlike the Wasserstein distance, which has been used for generative models such as WGAN \citep{arjovsky2017wasserstein} but whose estimation suffers from the curse of dimensionality, max-sliced and sliced Wasserstein enjoy parametric estimation rates which make them more suitable as discriminators.

The goal of this section is to show that 
$\mathcal{F}_1$ IPMs are equivalent to max-sliced Wasserstein distances up to a constant power, while sliced Wasserstein distances are similarly equivalent to a fixed-kernel IPM with a kernel that is slightly different from the 
$\mathcal{F}_2$ kernel. These bounds are helpful to get a quantitative understanding of how strong feature-learning and fixed-kernel IPMs are, and provide a novel bridge between sliced optimal transport and generative modeling discriminators.

\subsection{Spiked and sliced Wasserstein distances}
Throughout this section $k$ denotes an integer such that $1 \leq k \leq d$. The Stiefel manifold $\mathcal{V}_k$ is the set of matrices $U \in \R^{k \times d}$ such that $U U^\top = I_{k \times k}$ (i.e. the rows of $U$ are orthonormal).
We define the \textit{$k$-dimensional projection robust $p$-Wasserstein distance} between $\mu,\nu \in \mathcal{P}(\R^d)$ as 
\begin{align} \label{eq:def_k_dim_projection}
    \overline{\mathcal{W}}_{p,k}(\mu,\nu)^p = \max_{U \in \mathcal{V}_k} \min_{\pi \in \Gamma(\mu,\nu)} \int \|Ux-Uy\|^p d\pi(x,y), 
\end{align}
where $\Gamma(\mu,\nu)$ denotes the set of couplings between $\mu,\nu$, i.e. of measures $\mathcal{P}(K \times K)$ with projections $\mu$ and $\nu$.
This is the distance studied by \cite{nilesweed2019estimation} as a good estimator for the Wasserstein distance for a certain class of target densities with low dimensional structure.

The \textit{integral $k$-dimensional projection robust $p$-Wasserstein distance} between $\mu,\nu \in \mathcal{P}(\R^d)$ is defined as 
\begin{align} \label{eq:def_k_dim_int_projection}
    \underline{\mathcal{W}}_{p,k}(\mu,\nu)^p = \int_{\mathcal{V}_k} \bigg(\min_{\pi \in \Gamma(\mu,\nu)} \int \|Ux-Uy\|^p d\pi(x,y) \bigg) d\tau(U), 
\end{align}
where $\tau$ is the uniform measure over $\mathcal{V}_k$.
\cite{nadjahi2020statistical} studied statistical aspects of this distance in the case in which $k=1$, 
while \cite{lin2021projection} considers the case with general $k$. Notice trivially that $\overline{\mathcal{W}}_{p,k}(\mu,\nu) \geq \underline{\mathcal{W}}_{p,k}(\mu,\nu)$.

Sliced Wasserstein distances are spiked Wasserstein distances with $k=1$, but they were studied first chronologically \citep{bonneel2014sliced, kolouri2016sliced, kolouri2019generalized}. Namely, the \textit{sliced Wasserstein distance} is the integral 1-dimensional projection robust Wasserstein distance $\underline{\mathcal{W}}_{p,k}$, and the \textit{max-sliced Wasserstein distance} is the 1-dimensional projection robust Wasserstein distance $\overline{\mathcal{W}}_{p,k}$. 
Some arguments are easier for the case $k=1$ because the Stiefel manifold is the sphere $\mathbb{S}^{d-1}$.

\subsection{Results}
We prove in \autoref{thm:d_f1_spiked_w} that for $K = \{x \in \R^d | \|x\|_2 \leq 1 \} \times \{1\}$, for which the $\mathcal{F}_1$ space corresponds to overparametrized two-layer neural networks with bias, the $\mathcal{F}_1$ IPM can be upper and lower-bounded by the projection robust Wasserstein distance $\overline{\mathcal{W}}_{1,k}(\mu,\nu)$ up to a constant power (not depending on the dimension). 

\begin{restatable}{thm}{thmdfonespikedw} \label{thm:d_f1_spiked_w}
Let $\delta > 0$ be larger than a certain constant depending on $k$ and $\alpha$. Let $\sigma(x) = (x)_{+}^{\alpha}$ be the $\alpha$-th power of the ReLu activation function, where $\alpha$ is a non-negative integer. Let $\mu,\nu$ be Borel probability measures with support included in $\{x \in \R^d | \|x\|_2 \leq 1 \} \times \{1\}$. Let $d_{\mathcal{B}_{\mathcal{F}_1}}$ be as defined in \eqref{eq:def_f1_ipm} and $\overline{\mathcal{W}}_{1,k}$ as defined in \eqref{eq:def_k_dim_projection}. Then,
\begin{align} \label{eq:f_1_spiked}
    \delta \overline{\mathcal{W}}_{1,k}(\mu,\nu) \geq \delta d_{\mathcal{B}_{\mathcal{F}_1}}(\mu,\nu) \geq \overline{\mathcal{W}}_{1,k}(\mu,\nu) - 2 C(k,\alpha) \delta^{-\frac{1}{\alpha+(k-1)/2}} \log \left(\delta \right),
\end{align}
where $C(k,\alpha)$ is a constant that depends only on $k$ and $\alpha$.
If we optimize the lower bound in \eqref{eq:f_1_spiked} with respect to $\delta$, we obtain $\overline{\mathcal{W}}_{1,k}(\mu,\nu) \geq d_{\mathcal{B}_{\mathcal{F}_1}}(\mu,\nu) \geq \tilde{O}(\overline{\mathcal{W}}_{1,k}(\mu,\nu)^{\alpha + \frac{k+1}{2}})$
where $\tilde{O}$ hides log factors.
\end{restatable}

While for 
the $\mathcal{F}_2$ IPM the link with the sliced Wasserstein distance is not straightforward, it can be established when we switch from uniform $\tau$ to an alternative feature measure $\tilde{\tau}$. We define the class $\tilde{\mathcal{F}}_2$ of functions $\R^{d} \to \R$ as the RKHS associated with the following kernel 
\begin{align}
\begin{split}
    &\tilde{k}(x,y) = \int_{\mathbb{S}^d} \sigma(\langle (x,1), \theta \rangle) \sigma(\langle (y,1), \theta \rangle) \ d\tilde{\tau}(\theta) = \\ &\frac{1}{\pi}\int_{\mathbb{S}^{d-1}} \int_{-1}^{1} \sigma\left(\langle (x,1), (\sqrt{1-t^2} \xi, t) \rangle \right) \sigma\left(\langle (y,1), (\sqrt{1-t^2} \xi, t) \rangle \right) (1-t^2)^{-1/2} \ dt \ d\tau_{(d-1)}(\xi).
\end{split}
\end{align}
\begin{restatable}{prop}{proptildetau}\textnormal{(\textbf{$\tilde{\tau}$ as a rescaling of uniform measure})} \label{prop:tildetau}
The measure $d\tilde{\tau}(\sqrt{1-t^2} \xi, t) = \frac{1}{\pi} (1-t^2)^{-1/2} \ dt \ d\tau_{(d-1)}(\xi)$ is a probability measure. For comparison, the uniform measure over $\mathbb{S}^d$ can be written as $d\tau(\sqrt{1-t^2} \xi, t) = \frac{\Gamma((d+1)/2)}{\sqrt{\pi} \Gamma(d/2)} (1-t^2)^{\frac{d-1}{2}} \ dt \ d\tau_{(d-1)}(\xi)$.
\end{restatable}
That is, $\mathcal{F}_2$ and $\tilde{\mathcal{F}_2}$ are both fixed-kernel spaces with a similar kernel. They differ only in the weighing measure of the kernel; all the expressions which are valid in the $\mathcal{F}_2$ setting are also valid for $\tilde{\mathcal{F}_2}$ if we replace $\tau$ by $\tilde{\tau}$.
In analogy with the $\mathcal{F}_2$ IPM, the $\tilde{\mathcal{F}_2}$ IPM is given below.
\begin{align} \label{eq:def_tilde_f2}
d^2_{\mathcal{B}_{\tilde{\mathcal{F}}_2}}(\mu,\nu) = \int_{\mathbb{S}^d} \left( \int_K \sigma(\langle x, \theta \rangle) d(\mu-\nu)(x) \right)^2 d\tilde{\tau}(\theta).
\end{align}

Analogously to \autoref{thm:d_f1_spiked_w}, \autoref{thm:d_tilde_f2} establishes that the $\tilde{\mathcal{F}_2}$ IPM is upper and lower-bounded by the sliced Wasserstein distance $\underline{\mathcal{W}}_{1,1}(\mu,\nu)$ up to a constant power (not depending on the dimension). The reason to introduce the space $\tilde{\mathcal{F}_2}$ is that in the proof, the argument that makes the connection with the sliced Wasserstein distance requires the base measure of the kernel to be $\tilde{\tau}$ and does not work for $\tau$. However, we do not imply that a similar result for the $\mathcal{F}_2$ IPM is false.

\begin{restatable}{thm}{thmdtildeftwo} \label{thm:d_tilde_f2}
Let $\delta > 0$ be larger than a certain constant depending on $k$ and $\alpha$. Let $\sigma(x) = (x)_{+}^{\alpha}$ be the $\alpha$-th power of the ReLu activation function, where $\alpha$ is a non-negative integer. Let $\mu,\nu$ be Borel probability measures with support included in $\{x \in \R^d | \|x\|_2 \leq 1 \} \times \{1\}$. Let $d_{\mathcal{B}_{\tilde{\mathcal{F}}_2}}$ be as defined in \eqref{eq:def_tilde_f2} and $\underline{\mathcal{W}}_{1,1}$ as defined in \eqref{eq:def_k_dim_int_projection}. Then,
\begin{align} \label{eq:tilde_f2_lower}
    \delta d^{2/3}_{\tilde{\mathcal{F}}_2}(\mu,\nu) \geq \left(\frac{5}{12 \pi \alpha 2^{\alpha/2}} \right)^{1/3} \left( \underline{\mathcal{W}}_{1,1}(\mu,\nu) - 2 C(1,\alpha) \delta^{-\frac{1}{\alpha}} \log \left(\delta \right) \right).
\end{align}
and $ \pi d^2_{\mathcal{B}_{\tilde{\mathcal{F}}_2}}(\mu,\nu) \leq \underline{\mathcal{W}}_{1,1}(\mu,\nu)$. If we optimize the lower bound in \eqref{eq:tilde_f2_lower} with respect to $\delta$, we obtain $d^{2/3}_{\tilde{\mathcal{F}}_2}(\mu,\nu) \geq \tilde{O}(\underline{\mathcal{W}}_{1,1}(\mu,\nu)^{1+\alpha})$.
\end{restatable}


\section{Experiments} \label{sec:experiments}
To validate and clarify our findings, we perform experiments of the settings studied \autoref{sec:sep_f1_f2_ipm}, \autoref{sec:sep_sd} and \autoref{sec:bounds_f1_f2}. 
We use the ReLu activation function $\sigma(x) = (x)_{+}$, although remark that the results of \autoref{sec:sep_f1_f2_ipm} hold for a generic activation function, and the results of \autoref{sec:sep_sd} and \autoref{sec:bounds_f1_f2} hold for non-negative integer powers of the ReLu activation. The empirical estimates in the plots are detailed in \autoref{sec:details_exp}. They are averaged over 10 repetitions; the error bars show the maximum and minimum.
\begin{figure*}[ht!]
    \centering
    \includegraphics[width=.4\textwidth]{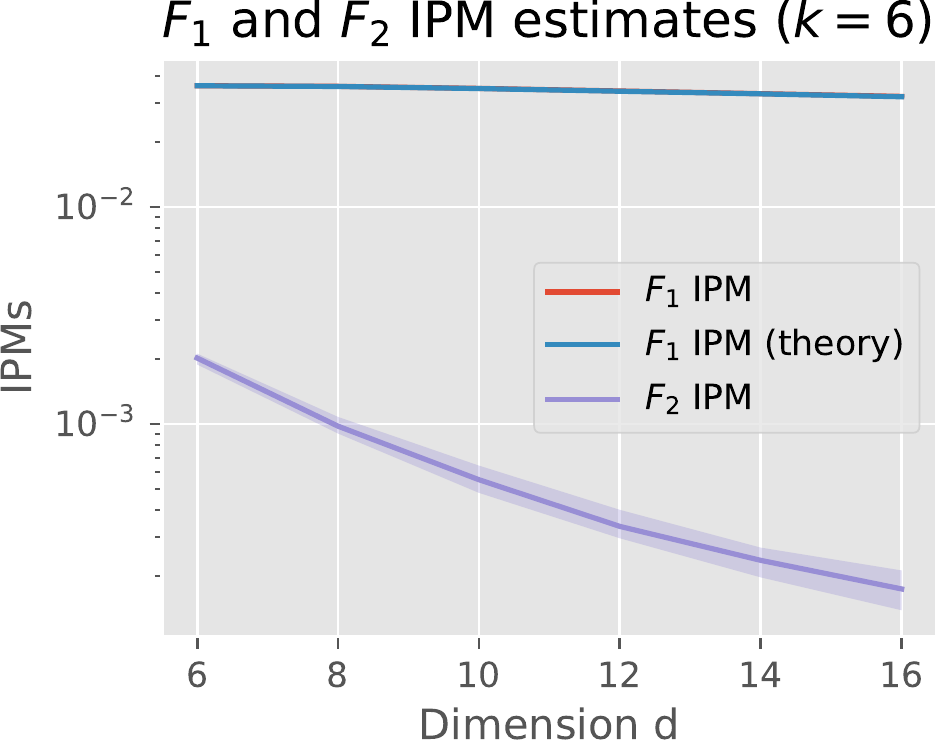}
    \includegraphics[width=.4\textwidth]{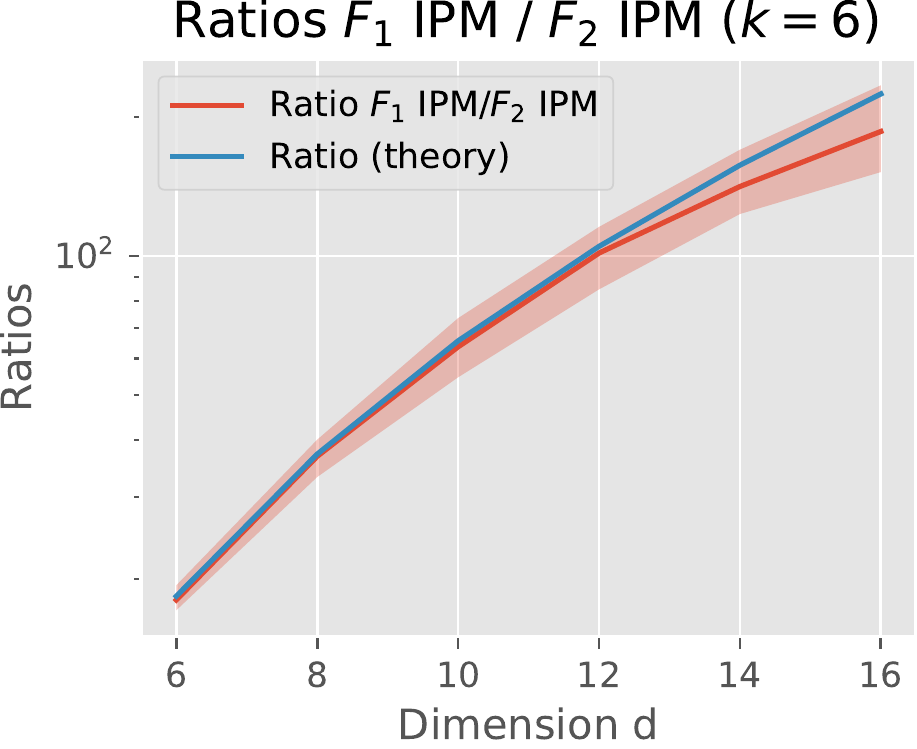}
    \caption{$\mathcal{F}_1$ and $\mathcal{F}_2$ IPM estimates for the pairs $\mu_d$ and $\nu_d$ defined in \eqref{eq:densities_mu_nu} for $k=6$ and varying dimension $d$. (Left) The blue and red curves (superposed) show two different estimates of the $\mathcal{F}_1$ IPM. The purple curve shows estimates for the $\mathcal{F}_2$ IPM. (Right) The blue curve shows the theoretical ratio between the $\mathcal{F}_1$ and the $\mathcal{F}_2$ IPMs (see equation \eqref{eq:distances_ratio0}). The red curve shows an empirical estimate of the ratio obtained by dividing the IPM estimates. 4400 million samples of $\mu_d$ and $\nu_d$ are used.}
    \label{fig:f1_f2_ipm_separation_figure}
\end{figure*}
\vskip -0.1 in
\textbf{Separation between $\mathcal{F}_1$ and $\mathcal{F}_2$ IPMs.} \autoref{fig:f1_f2_ipm_separation_figure} shows $\mathcal{F}_1$ and $\mathcal{F}_2$ IPM estimates for the pairs $\mu_d$ and $\nu_d$ defined in \eqref{eq:densities_mu_nu} for the Legendre polynomial of degree $k=6$ and varying dimension $d$, and its ratios. We observe that while the $\mathcal{F}_1$ IPM remains nearly constant in the dimension, the $\mathcal{F}_2$ IPM experiences a significant decrease. The ratios between IPMs closely track those predicted by our \autoref{thm:separation_ipm}, the mismatch being due to the overestimation of the $\mathcal{F}_2$ IPM caused by statistical errors. We were constrained in the values of 
$k$ and 
$d$ that we could choose; when the $\mathcal{F}_2$ IPM is small, which is the case when $k$ and/or 
$d$ are large, we need a high number of samples from the distributions 
$\mu_d, \nu_d$ to make the statistical error smaller than 
$d_{\mathcal{B}_{\mathcal{F}_2}}(\mu_d,\nu_d)$ and get a good estimate.

\textbf{Separation between $\mathcal{F}_1$ and $\mathcal{F}_2$ SDs.} \autoref{fig:f1_f2_sd_separation_figure} shows $\mathcal{F}_1$ and $\mathcal{F}_2$ SD estimates for the pairs $\mu_d$ and $\nu_d$ defined in equation \autoref{sec:separation_sd} for the Legendre polynomial of degree $k=5$ and varying dimension $d$, and its ratios. The fact that the empirical ratio is significantly above the theoretical lower bound indicates that our lower bound (although exponential) is not tight. This can be guessed by looking at the slackness in the inequalities of \autoref{lem:f1_sd_lower} and \autoref{lem:sd_f2_upper}.

\begin{figure*}[ht!]
    \centering
    \includegraphics[width=.4\textwidth]{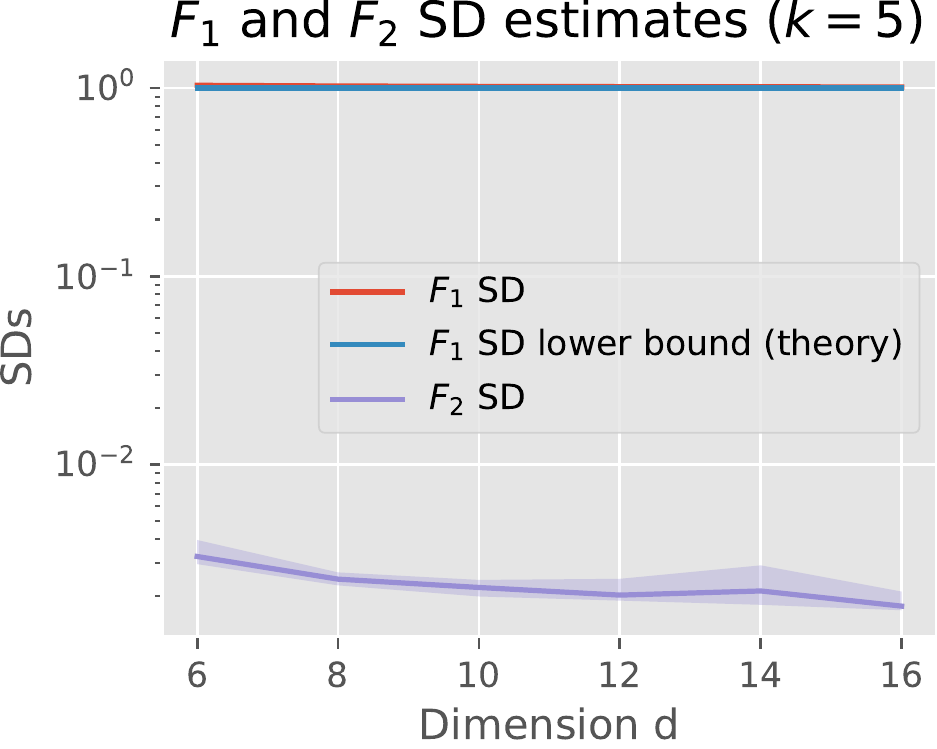}
    \includegraphics[width=.4\textwidth]{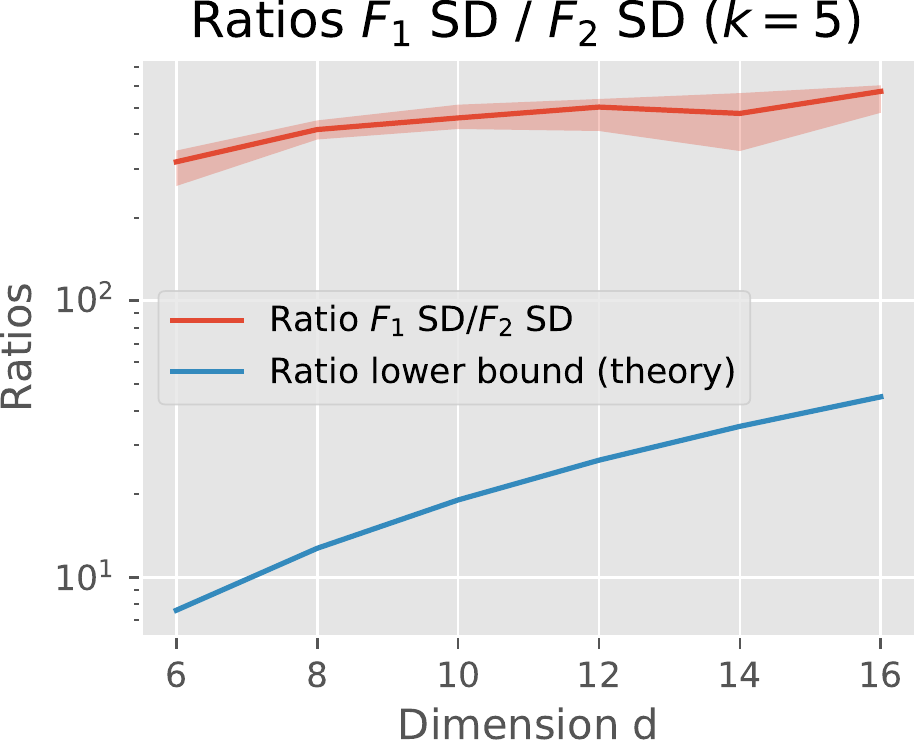}
    \caption{$\mathcal{F}_1$ and $\mathcal{F}_2$ SD estimates for the pairs $\mu_d$ and $\nu_d$ defined in \autoref{sec:separation_sd} for $k=5$ and varying dimension $d$. (Left) The red curve shows an empirical estimate of the $\mathcal{F}_1$ SD, the blue curve shows a theoretical lower bound (\autoref{lem:f1_sd_lower} in \autoref{sec:proofs_sep_sd}) on the $\mathcal{F}_1$ SD, the purple curve shows an estimate of the $\mathcal{F}_2$ SD. (Right) The blue curve represents the theoretical lower bound on the ratio between the $\mathcal{F}_1$ and the $\mathcal{F}_2$ SDs (see equation \eqref{eq:sd_sep}),  while the red curve shows an empirical estimate of the ratio obtained by dividing the SD estimates. 30 million samples are used.}
    \label{fig:f1_f2_sd_separation_figure}
\end{figure*}

\textbf{ $\mathcal{F}_1$, $\mathcal{F}_2$, $\tilde{\mathcal{F}}_2$ IPMs versus max-sliced and sliced Wasserstein.} \autoref{fig:f1_f2_wasserstein_figure} shows several metrics between a standard multivariate Gaussian and a Gaussian with unit variance in all directions except for one of smaller variance $0.1$, in varying dimensions. We observe that while the $\mathcal{F}_1$ IPM and the max-sliced Wasserstein distance are constant, the $\mathcal{F}_2$, $\tilde{\mathcal{F}}_2$ IPMs and the sliced Wasserstein distance decrease. For high dimensions they match the corresponding distances between two datasets of standard multivariate Gaussian, which means that the statistical noise precludes discrimination in these metrics.  

    

\begin{figure}[ht!]
\begin{minipage}[c]{7cm}
\includegraphics[width=0.95\textwidth]{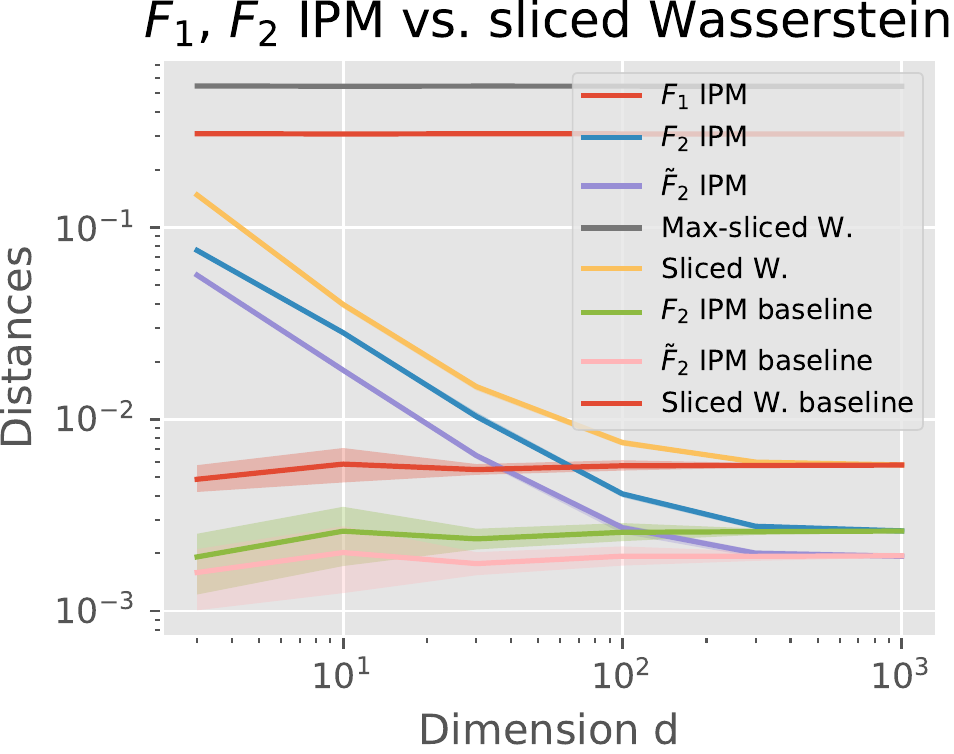}
\end{minipage}
\begin{minipage}[c]{7cm}
\caption{For varying dimension $d$, we plot $\mathcal{F}_1$, $\mathcal{F}_2$, $\tilde{\mathcal{F}}_2$ IPM, sliced and max-sliced Wasserstein estimates between a standard multivariate Gaussian and a Gaussian with unit variance in all directions except for one of smaller variance $0.1$. The estimates are computed using 100000 samples of each distribution. For comparison, the same estimates are shown between a standard multivariate Gaussian and itself, using two different sets of 100000 samples.}
\label{fig:f1_f2_wasserstein_figure}
\end{minipage}
\end{figure}
\vskip -0.2in
\section{Conclusions and discussion}
We have shown pairs of distributions over hyperspheres for which the $\mathcal{F}_1$ IPM and SD are exponentially larger than the $\mathcal{F}_2$ IPM and SD. In parallel, we have also provided links between the $\mathcal{F}_1$ IPM and max-sliced Wasserstein distance, and between the $\tilde{\mathcal{F}}_2$ IPM and the sliced Wasserstein distance. The densities of the distributions constructed in Sections \autoref{sec:sep_f1_f2_ipm} and \autoref{sec:sep_sd} are based on Legendre harmonics of increasing degree. Keeping in mind that spherical harmonics are the Fourier basis for $L^2(\mathbb{S}^{d-1})$ (in the sense that they constitute an orthonormal basis of eigenvalues of the Laplace-Beltrami operator), one can infer a simple overarching idea from our constructions: ``$\mathcal{F}_1$ discriminators are better than $\mathcal{F}_2$ discriminators at telling apart distributions whose densities have only high frequency differences''. It would be interesting to develop this intuition into a more general theory. Another avenue of future work is to understand how deep discriminators perform versus shallow ones, in analogy with the work of \cite{eldan2016the} for regression.

\paragraph{Acknowledgements.} We thank Joan Bruna for useful discussions. CD acknowledges partial support by ``la Caixa'' Foundation (ID 100010434), under agreement LCF/BQ/AA18/11680094.

\newpage
\bibliography{refs,simplex,biblio}

\appendix

\include{appendix}

\end{document}

%% file: preamble.tex
\usepackage{amsmath}
\usepackage{amssymb}
\usepackage{graphicx}
\usepackage{mathtools}
\usepackage{amsfonts}
\usepackage{amsthm,bm}
\usepackage{color}
\usepackage{enumitem}
\usepackage{dsfont}
\usepackage{pgfplots}
\pgfplotsset{width=10cm,compat=1.9}
 \usepackage{pifont}
\usepackage{thmtools} 
\usepackage{thm-restate}
\usepackage{sidecap}

\newcommand{\R}{\mathbb{R}}

\renewcommand{\phi}{\varphi}

\graphicspath {{figures/}}

\usepackage{hyperref}
\usepackage{caption}
\usepackage[super]{nth}
\newtheorem{lemma}{Lemma}

\mathtoolsset{showonlyrefs}

%% file: appendix.tex
\section{$\mathcal{F}_1$ and $\mathcal{F}_2$ IPMs} \label{sec:app_f1_f2}

\begin{lemma} \label{lem:f1_exp}
The $\mathcal{F}_1$ IPM can be written as
\begin{align}
    d_{\mathcal{B}_{\mathcal{F}_1}}(\mu,\nu) = \sup_{\theta \in \mathbb{S}^d} \left| \int_K \sigma(\langle x, \theta \rangle) d(\mu-\nu)(x) \right|.
\end{align}
\end{lemma}
\begin{proof}
\begin{align}
\begin{split} \label{eq:d_f1}
    d_{\mathcal{B}_{\mathcal{F}_1}}(\mu,\nu) &= \sup_{f\in \mathcal{B}_{\mathcal{F}_1}}  \mathbb{E}_{x\sim \mu } f(x) - \mathbb{E}_{x\sim \nu} f(x) \\ &= \sup_{\|\gamma\|_{\text{TV}} \leq 1}  \mathbb{E}_{x\sim \mu } \int_{\mathbb{S}^d} \sigma(\langle x, \theta \rangle) d\gamma(\theta) - \mathbb{E}_{x\sim \nu} \int_{\mathbb{S}^d} \sigma(\langle x, \theta \rangle) d\gamma(\theta) \\ &= \sup_{\|\mu\|_{\text{TV}} \leq 1}  \int_{\mathbb{S}^d} \left( \int_K \sigma(\langle x, \theta \rangle) d(\mu-\nu)(x) \right) d\gamma(\theta) \\ &= \sup_{\theta \in \mathbb{S}^d} \left| \int_K \sigma(\langle x, \theta \rangle) d(\mu-\nu)(x) \right|
\end{split}
\end{align}
In the last equality we have used that the set $\{ \mu \in \mathcal{M}(\mathbb{S}^d) \ | \ \|\mu\|_{\text{TV}} \leq 1 \}$ can be seen as the convex hull of $\{ \delta_{\theta} | \theta \in \mathbb{S}^d \}$, which means that optimizing a convex function over one set and the other yields the same optimal value.
\end{proof}

\begin{lemma} \label{lem:f2_exp}
The $\mathcal{F}_2$ IPM can be written as
\begin{align}
    d^2_{\mathcal{B}_{\mathcal{F}_2}}(\mu,\nu) = \int_{\mathbb{S}^d} \left( \int_K \sigma(\langle x, \theta \rangle) d(\mu-\nu)(x) \right)^2 d\tau(\theta).
\end{align}
\end{lemma}
\begin{proof}
\begin{align}
\begin{split} \label{eq:d_f2}
    d^2_{\mathcal{B}_{\mathcal{F}_2}}(\mu,\nu) &= \left(\sup_{f\in \mathcal{B}_{\mathcal{F}_2}}  \mathbb{E}_{x\sim \mu} f(x) - \mathbb{E}_{x\sim \nu} f(x) \right)^2 = \left(\sup_{\|f\|_{\mathcal{F}_2}  \leq 1} \int_K \langle f, k(x, \cdot) \rangle_{\mathcal{F}_2} d(\mu-\nu)(x) \right)^2 \\ &= \left(\sup_{\|f\|_{\mathcal{F}_2}  \leq 1} \left\langle f, \int_K k(x, \cdot) d(\mu-\nu)(x) \right\rangle_{\mathcal{F}_2} \right)^2 = \left\| \int_K k(x, \cdot) d(\mu-\nu)(x) \right\|_{\mathcal{F}_2}^2 \\ &= \iint_{K \times K} k(x,y) \ d(\mu-\nu)(x) d(\mu-\nu)(y) = \iint_{K \times K} k(x,y) \ d(\mu-\nu)(x) d(\mu-\nu)(y) \\ &= \iint_{K \times K} \int_{\mathbb{S}^d} \sigma(\langle x, \theta \rangle) \sigma(\langle y, \theta \rangle) d\tau(\theta) d(\mu-\nu)(x) d(\mu-\nu)(y) \\ &= \int_{\mathbb{S}^d} \left( \int_K \sigma(\langle x, \theta \rangle) d(\mu-\nu)(x) \right)^2 d\tau(\theta).
\end{split}
\end{align}
\end{proof}

\section{Are the $\mathcal{F}_1$ and $\mathcal{F}_2$ IPMs distances?} \label{sec:distances}

In the following we consider unitary Fourier transforms with angular frequency: for $f \in L^1(\R^{d+1})$, we have $\hat{f}(\omega) = \frac{1}{(2\pi)^{(d+1)/2}} \int_{\R^{d+1}} f(x) e^{-i \langle \omega, x \rangle} dx$ and if $\hat{f} \in L^1(\R^{d+1})$, then $f(x) = \frac{1}{(2\pi)^{(d+1)/2}} \int_{\R^d} \hat{f}(\omega) e^{-i \langle \omega, x \rangle} dx$. We denote the space of tempered distributions on $\R^{d+1}$ as $\mathcal{S}'(\R^{d+1})$, i.e., as the dual of the space $\mathcal{S}(\R^{d+1})$ of Schwartz functions, which are functions in $\mathcal{C}^{\infty}(\R^{d+1})$ whose derivatives of any order decay faster than polynomials of all orders. Functions that grow no faster than polynomials can be embedded in $\mathcal{S}'(\R^{d+1})$ by defining $g(\phi) := \int_{\R^{d+1}} \varphi(x) g(x) \ dx$ for any $\varphi \in \mathcal{S}(\R^{d+1})$. The Fourier transform of a tempered distribution $T$ can be defined as the tempered distribution $\hat{T}$ that acts on $\varphi \in \mathcal{S}(\R^{d+1})$ as $\langle \hat{T}, \varphi \rangle = \langle T, \hat{\varphi} \rangle$. Fourier transforms of two-layer neural networks have been used in prior works, e.g. \cite{venturi2021depth}.

\begin{lemma} \label{lem:ft_g}
Let $\hat{\sigma} \in \mathcal{S}'(\R)$ be the Fourier transform of the activation function $\sigma : \R \rightarrow \R$ in the sense of tempered distributions. Let $g(\theta) = \int_K \sigma(\langle x, \theta \rangle) d(\mu-\nu)(x)$. The Fourier transform of $g \in \mathcal{S}'(\R^{d+1})$ in the sense of tempered distributions is the tempered distribution $\hat{g}$ defined as
\begin{align}
    \left\langle \hat{g}, \phi \right\rangle = (2\pi)^{d/2} \int_K \langle \hat{\sigma}, \phi(\cdot x) \rangle \ d(\mu-\nu)(x)
\end{align}
for any $\phi \in \mathcal{S}(\R^{d+1})$.
\end{lemma}
\begin{proof}
The Fourier transform of the tempered distribution $T_{x}$ defined as $\langle T_{x}, \phi \rangle = \int_{\R^{d+1}} \phi(\theta) \sigma(\langle x, \theta \rangle) \ d\theta$ is $\hat{T}_{x}$ defined as
\begin{align}
\begin{split}
    &\langle \hat{T}_{x}, \phi \rangle = \int_{\R^{d+1}} \left( \frac{1}{(2\pi)^{(d+1)/2}} \int_{\R^{d+1}} \phi(\theta) e^{-i \langle \omega, \theta \rangle} d\theta \right) \sigma(\langle x, \omega \rangle) \ d\omega \\ &= \int_{\text{span}(x)} \int_{\text{span}(x)^{\perp}}  \left( \frac{1}{(2\pi)^{(d+1)/2}} \int_{\text{span}(x)^{\perp}} \left( \int_{\text{span}(x)} \phi(\theta) e^{-i \langle \omega_{x}, \theta_{x} \rangle} d\theta_{x} \right) e^{-i \langle \omega_{x^{\perp}}, \theta_{x^{\perp}} \rangle} d\theta_{x^{\perp}} \right) d\omega_{x^{\perp}} \sigma(\langle x, \omega_{x} \rangle) d\omega_{x}
    \\ &= \frac{1}{(2\pi)^{1/2}} \int_{\text{span}(x)} (2\pi)^{d/2} \int_{\text{span}(x)} \phi(\theta_{x}) e^{-i \langle \omega_{x}, \theta_{x} \rangle} d\theta_{x} \sigma(\langle x, \omega_{x} \rangle) d\omega_{x}
    \\ &= \frac{1}{(2\pi)^{1/2}} \int_{\R} (2\pi)^{d/2} \int_{\R} \phi(t x) e^{-i \omega t} dt \sigma(\omega) d\omega
    \\ &= (2\pi)^{d/2} \langle \hat{\sigma}, \phi( \cdot x) \rangle 
\end{split}
\end{align}
Here, the first equality holds because by definition, $\langle \hat{T}_{x}, \phi \rangle = \langle T_{x}, \hat{\phi} \rangle$. In the second equality, we rewrite $\R^{d+1} = \text{span}(x) + \text{span}(x)^{\perp}$ and we use Fubini's theorem twice. In the third equality we make the following argument: denoting $h(\theta_{x^{\perp}}, \omega_x) = \int_{\text{span}(x)} \phi(\theta_{x^{\perp}} + \theta_x) e^{-i \langle \omega_{x}, \theta_{x} \rangle} d\theta_{x}$, we have that 
$\int_{\text{span}(x)^{\perp}} \left( \int_{\text{span}(x)^{\perp}} h(\theta_{x^{\perp}}, \omega_x) e^{-i \langle \omega_{x^{\perp}}, \theta_{x^{\perp}} \rangle} d\theta_{x^{\perp}} \right) d\omega_{x^{\perp}} = (2\pi)^{d/2} \int_{\text{span}(x)^{\perp}} \hat{h}(\omega_{x^{\perp}}, \omega_x) d\omega_{x^{\perp}} = (2\pi)^{d} h(0, \omega_x) = (2\pi)^{d} \int_{\text{span}(x)} \phi(\theta_x) e^{-i \langle \omega_{x}, \theta_{x} \rangle} d\theta_{x}$.

Notice that we can write $g$ as a tempered distribution as $g = \int_K T_x \ d(\mu-\nu)(x)$. Thus, by linearity of the Fourier transform, we have that
\begin{align}
    \left\langle \hat{g}, \phi \right\rangle = (2\pi)^{d/2} \int_K \langle \hat{\sigma}, \phi(\cdot x) \rangle \ d(\mu-\nu)(x)
\end{align}
for any $\phi \in \mathcal{S}(\R^{d+1})$.
\end{proof}
We compute $\hat{\sigma}$ for the specific case in which $\sigma: \R \rightarrow \R$ is an $\alpha$-positive homogeneous activation function, i.e. $\sigma(x) = a (x)_{+}^{\alpha} + b (-x)_{+}^{\alpha}$ for some $a, b \in \R$ (equation \eqref{eq:alpha_positive}). 
It is known \citep{erdelyi1954tables, kammler2000afirst} that the Fourier transform of the Heaviside step function $u : \R \to \R$, defined as $u(x) = 1$ if $x \geq 1$ and $u(x) = 0$ if $x < -1$, is $\hat{u}(\omega) = \sqrt{\frac{\pi}{2}} \left( \text{p.v.}\left[\frac{1}{i\pi \omega} \right] + \delta(\omega) \right)$. Here  $\text{p.v.}\left[\frac{1}{\omega} \right] \in \mathcal{S}'(\R)$ is a Cauchy principal value, defined as $\text{p.v.}\left[\frac{1}{\omega} \right](\phi) = \lim_{\epsilon \to 0} \int_{\R \setminus [-\epsilon,\epsilon]} \frac{1}{\omega} \phi(\omega) \ d\omega$ for any $\phi \in \mathcal{S}(\R)$.

Moreover, for any tempered distribution $f \in \mathcal{S}'(\R)$, for $\alpha \geq 0$ integer, the Fourier transform of $x^{\alpha} f(x)$ is $i^{\alpha} \frac{d^{\alpha} \hat{f}(\omega)}{d\omega^{\alpha}}$, where the derivative of a tempered distribution is defined in the weak sense: $\langle \frac{df}{d\omega}, \phi \rangle = - \langle f, \frac{d\phi}{d\omega} \rangle$.
Since $\sigma(x) = (a-(-1)^{\alpha} b) (x)_{+}^{\alpha} + (-1)^{\alpha} b x^{\alpha}$, we have that 
\begin{align}
\begin{split} \label{eq:hat_sigma}
\hat{\sigma}(\omega) &= (a-(-1)^{\alpha} b) i^{\alpha} \sqrt{\frac{\pi}{2}} \frac{d^{\alpha}}{d\omega^{\alpha}} \left( \text{p.v.} \left[\frac{1}{i\pi \omega} \right] + \delta(\omega) \right) + (-1)^{\alpha} b i^{\alpha} \frac{d^{\alpha}}{d\omega^{\alpha}} \delta(\omega) \\ &= A \frac{d^{\alpha}}{d\omega^{\alpha}} \left(\text{p.v.} \left[\frac{1}{i\pi \omega}\right] \right) + B \frac{d^{\alpha}}{d\omega^{\alpha}} \delta(\omega), 
\end{split}
\end{align}
where $A = i^{\alpha - 1} \frac{\alpha!}{\sqrt{2 \pi}} (a-(-1)^{\alpha} b)$ and $B = i^{\alpha} \sqrt{\frac{\pi}{2}} \left(a - (-1)^{\alpha} b \right) + (-i)^{\alpha} b$.

\begin{lemma}[Riesz-Markov theorem] \label{lem:riesz}
Let $X$ be a locally compact Hausdorff space. For any continuous linear functional $\psi$ on $C_0(X)$, there is a unique regular countably additive complex Borel measure $\mu$ on $X$ such that $\forall f \in C_{0}(X), \ \psi (f)= \int_{X}f(x) \,d\mu(x).$
The norm of $\psi$ as a linear functional is the total variation of $\mu$, that is
$\|\psi\| = |\mu|(X)$.
\end{lemma}

\begin{restatable}{thm}{thmipmdist} \label{thm:ipmdistbias}
    Let $\sigma: \R \rightarrow \R$ be an $\alpha$-positive homogeneous activation function of the form \eqref{eq:alpha_positive} and assume that $(-1)^{\alpha} a - b \neq 0$. Then, for $K = \R^{d} \times \{1\}$ both the $\mathcal{F}_1$ and the $\mathcal{F}_2$ IPMs are distances.
\end{restatable}
\begin{proof}
By \autoref{lem:f1_exp} and \autoref{lem:f2_exp} we have $d_{\mathcal{B}_{\mathcal{F}_1}}(\mu,\nu) = \sup_{\theta \in \mathbb{S}^d} \left| \int_K \sigma(\langle x, \theta \rangle) d(\mu-\nu)(x) \right|$ and $d^2_{\mathcal{B}_{\mathcal{F}_2}}(\mu,\nu) = \int_{\mathbb{S}^d} \left( \int_K \sigma(\langle x, \theta \rangle) d(\mu-\nu)(x) \right)^2 d\tau(\theta)$, which means that both are distances if the function $g\rvert_{\mathbb{S}^d} : \mathbb{S}^d \rightarrow \R$ defined as $g\rvert_{\mathbb{S}^d}(\theta) = \int_K \sigma(\langle x, \theta \rangle) d(\mu-\nu)(x)$ is different from zero in the $L^2$ sense when $\mu \neq \nu$. Since $\sigma$ is $\alpha$-positive homogeneous, the $\alpha$-positive homogeneous extension $g : \R^{d+1} \rightarrow \R$ of $g\rvert_{\mathbb{S}^d}$ fulfills 
\begin{align}
    g(\theta) = \int_K \sigma(\langle x, \theta \rangle) d(\mu-\nu)(x) = \|\theta\|_2^\alpha \int_K \sigma(\langle x, \theta/\|\theta\|_2 \rangle) d(\mu-\nu)(x) = \|\theta\|_2^\alpha g(\theta/\|\theta\|_2).
\end{align}
Thus, $g\rvert_{\mathbb{S}^d}$ is different from zero in the $L^2$ sense if and only if $g$ is. And $g$ is different from zero if and only if its Fourier transform $\hat{g}$ in the sense of tempered distributions is different from zero (this follows from $\langle \hat{g}, \phi \rangle = \langle g, \hat{\phi} \rangle$ for all $\phi \in \mathcal{S}(\R)$). By \autoref{lem:ft_g}, we have that
\begin{align} \label{eq:hatgphi}
    \left\langle \hat{g}, \phi \right\rangle = (2\pi)^{d/2} \int_K \langle \hat{\sigma}, \phi(\cdot x) \rangle \ d(\mu-\nu)(x).
\end{align}
By \eqref{eq:hat_sigma}, we have that
\begin{align} \label{eq:hatsigmaphi}
    \langle \hat{\sigma}, \phi(\cdot x) \rangle = (-1)^{\alpha} A \left( \text{p.v.}\left[ \frac{1}{t} \right] \right) \left( \frac{d^{\alpha}}{dt^{\alpha}} \phi(t x) \right) + (-1)^{\alpha} B \frac{d^{\alpha}}{dt^{\alpha}} \phi(t x)\bigg\rvert_{t = 0},  
\end{align}
which means that
\begin{align}
    \left\langle \hat{g}, \phi \right\rangle = (-1)^{\alpha} (2\pi)^{d/2} \int_K \left( A \left( \text{p.v.}\left[ \frac{1}{t} \right] \right) \left( \frac{d^{\alpha}}{dt^{\alpha}} \phi(t x) \right) + B \frac{d^{\alpha}}{dt^{\alpha}} \phi(t x)\bigg\rvert_{t = 0} \right) \ d(\mu-\nu)(x).
\end{align}
Suppose that $\mu \neq \nu$. Since $\mu$ and $\nu$ are Borel measures, they are are regular, and thus $\mu-\nu$ is also regular. By \autoref{lem:riesz}, $\mu-\nu$ can be identified univocally with an element of the dual space $C_0(\R^{d+1})'$ of the space $C_0(\R^{d+1})$ of continuous functions on $\R^{d+1}$ that vanish at infinity. Since $\mu - \nu \neq 0$, there must exist $\phi \in C_0(\R^{d+1})$ such that $\int_K \phi(x) d(\mu-\nu)(x) \neq 0$. Multiplying by the indicator function of a well chosen compact set and using a mollifier sequence, we can further assume that $\phi \in C_c^{\infty}(\R^{d+1}) \subseteq \mathcal{S}(\R^{d+1})$. 

Now, let $\eta$ be a $C_c^{\infty}(\R)$ function such that $\int_{R} \eta(t) \ dt = 1$ and the support of $\eta$ is compact and contained in $[1/2,+\infty)$. We define the function $\psi \in C_c^{\infty}(\R^{d+1}) \subseteq \mathcal{S}(\R^{d+1})$ as 
\begin{align}
    \psi(tx) = \alpha! \ t^{\alpha+1} \phi(x) \eta(t), \quad \forall x \in \R^d \times \{1\}, \forall t \in \R. 
\end{align}
Remark that
\begin{align} \label{eq:psiderivative}
    \frac{d^{\alpha}}{dt^{\alpha}} \psi(t x)\bigg\rvert_{t = 0} = 0,
\end{align}
because $\psi$ is equal to zero in a neighborhood of the origin.
Also, for all $x \in K$,
\begin{align} 
\begin{split} \label{eq:psiphi}
    &(-1)^{\alpha} \left( \text{p.v.}\left[ \frac{1}{t} \right] \right) \left( \frac{d^{\alpha}}{dt^{\alpha}} \psi(t x) \right) = (-1)^{\alpha} \int_{\R} \frac{1}{t} \frac{d^{\alpha}}{dt^{\alpha}} \left( \alpha! \ t^{\alpha+1} \phi(x) \eta(t) \right) \ dt
    \\ &= \int_{\R} \frac{1}{t^{\alpha+1}} t^{\alpha+1} \phi(x) \eta(t) \ dt = \phi(x) \int_{\R} \eta(t) \ dt = \phi(x).
\end{split}
\end{align}
In the first equality we have used that $\lim_{\epsilon \to 0} \int_{\R \setminus [-\epsilon, \epsilon]} \frac{1}{t} \frac{d^{\alpha}}{dt^{\alpha}} \psi(t x) \ dt = \int_{\R} \frac{1}{t} \frac{d^{\alpha}}{dt^{\alpha}} \psi(t x) \ dt$, again because $\psi$ is equal to zero in a neighborhood of the origin. 

Notice that since we have assumed that $(-1)^{\alpha}a - b \neq 0$, we have $A \neq 0$. Hence,
\begin{align}
\begin{split}
    0 &\neq (2\pi)^{d/2} A \int_K \phi(x) d(\mu-\nu)(x) \\ &= 
    (-1)^{\alpha} (2\pi)^{d/2} \int_K \left( A \left( \text{p.v.}\left[ \frac{1}{t} \right] \right) \left( \frac{d^{\alpha}}{dt^{\alpha}} \psi(t x) \right) + B \frac{d^{\alpha}}{dt^{\alpha}} \psi(t x)\bigg\rvert_{t = 0} \right) \ d(\mu-\nu)(x)
    = \langle \hat{g}, \psi \rangle
\end{split}
\end{align}
In the first equality, we have used \eqref{eq:psiphi} and \eqref{eq:psiderivative}. The last equality follows from \eqref{eq:hatgphi} and \eqref{eq:hatsigmaphi}.
We have constructed a function $\psi \in \mathcal{S}(\R^{d+1})$ for which $\hat{g}$ does not evaluate to zero, implying that $\hat{g} \neq 0$ and concluding the proof.
\end{proof}

\begin{restatable}{thm}{thmipmdistsphere} \label{thm:ipmdistsphere}
    Let $\sigma: \R \rightarrow \R$ be an $\alpha$-positive homogeneous activation function of the general form \eqref{eq:alpha_positive}. Then, for $K = \mathbb{S}^d$ both the $\mathcal{F}_1$ and the $\mathcal{F}_2$ IPMs are not distances.
\end{restatable}
\begin{proof}
Since by \autoref{lem:f1_exp} and \autoref{lem:f2_exp} we have $d_{\mathcal{B}_{\mathcal{F}_1}}(\mu,\nu) = \sup_{\theta \in \mathbb{S}^d} \left| \int_K \sigma(\langle x, \theta \rangle) d(\mu-\nu)(x) \right|$ and $d^2_{\mathcal{B}_{\mathcal{F}_2}}(\mu,\nu) = \int_{\mathbb{S}^d} \left( \int_K \sigma(\langle x, \theta \rangle) d(\mu-\nu)(x) \right)^2 d\tau(\theta)$, it suffices to see that the function $g\rvert_{\mathbb{S}^d}(\theta) = \int_K \sigma(\langle x, \theta \rangle) d(\mu-\nu)(x)$ can be zero in the $L^2$ sense for some pairs $\mu \neq \nu$. We want to study the kernel of the map $\mu \mapsto \int_K \sigma(\langle x, \cdot \rangle) d(\mu-\nu)(x)$. Notice that
\begin{align}
\begin{split}
    &\sigma(\langle x, \theta \rangle) + (-1)^{\alpha} \sigma(\langle -x, \theta \rangle) = a (\langle x, \theta \rangle)_{+}^{\alpha} + b (-\langle x, \theta \rangle)_{+}^{\alpha} + (-1)^{\alpha} (a (\langle -x, \theta \rangle)_{+}^{\alpha} + b (-\langle -x, \theta \rangle)_{+}^{\alpha}) \\ &= (a + (-1)^{\alpha} b) ((\langle x, \theta \rangle)_{+}^{\alpha} + (-1)^{\alpha}(-\langle x, \theta \rangle)_{+}^{\alpha}) = (a + (-1)^{\alpha} b) \langle x, \theta \rangle^{\alpha}.
\end{split}
\end{align}
If $\gamma \in \mathcal{M}(K)$ is an even (i.e., $(x \mapsto -x)_{\#} \gamma = \gamma$) or odd (i.e., $(x \mapsto -x)_{\#} \gamma = -\gamma$) signed measure of the same parity as $\alpha$, this implies that 
\begin{align}
\int_{K} \sigma(\langle x, \theta \rangle) d\gamma(x) = \frac{a + (-1)^{\alpha} b}{2} \int_{K} \langle x, \theta \rangle^{\alpha} \ d|\gamma|(x),
\end{align}
which is a polynomial of degree $\alpha$ on $\theta$. Consider the linear map $L : \mathcal{M}(K) \mapsto C(\mathbb{S}^d)$ defined as $\gamma \mapsto \int_K \sigma(\langle x, \cdot \rangle) d\gamma(x)$. Since $L$ restricted to the measures of the parity of $\alpha$ has an infinite-dimensional domain and a finite-dimensional image, it must have an infinite-dimensional kernel. 
\begin{itemize}[leftmargin=5.5mm]
\item For the case $\alpha$ odd, if $\gamma \in \mathcal{M}(K)$ is an odd measure belonging to the kernel of $L$ with total variation norm ${\|\gamma\|}_{\text{TV}} = 2$ and such that $\gamma = \gamma_{+} - \gamma_{-}$ with $\gamma_{+}, \gamma_{-}$ non-negative, then choosing $\mu = \gamma_{+}$ and $\nu = \gamma_{-}$, we have that $d_{\mathcal{B}_{\mathcal{F}_1}}(\mu,\nu) = d_{\mathcal{B}_{\mathcal{F}_2}}(\mu,\nu) = 0$. 
\item For the case $\alpha$ even, let $\gamma \in \mathcal{M}(K)$ be an even measure belonging to the kernel of $L$ with total variation norm ${\|\gamma\|}_{\text{TV}} = 2$. We must have that $\int_K d\gamma = 0$, because denoting by $\tau$ the uniform distribution over $\mathbb{S}^d$, we have
\begin{align}
    0 = \int_{\mathbb{S}^d} L\gamma (\theta) \ d\tau(\theta) = \int_{\mathbb{S}^d} \int_K \sigma(\langle x, \theta \rangle) \ d\gamma(x) \ d\tau(\theta) = \int_{\mathbb{S}^d} \sigma(\langle x', \theta \rangle) \ d\tau(\theta) \int_K \ d\gamma, 
\end{align}
for all $x' \in K$, and $\int_{\mathbb{S}^d} \sigma(\langle x', \theta \rangle) \ d\tau(\theta)$ is a strictly positive quantity. In the last equality we used Fubini's theorem. Thus, the non-negative components $\gamma_{+}, \gamma_{-}$ of the decomposition $\gamma = \gamma_{+} - \gamma_{-}$ must fulfill $\int_K d\gamma_{+} = \int_K d\gamma_{-} = 1$. Hence, choosing $\mu = \gamma_{+}$ and $\nu = \gamma_{-}$, we have that $d_{\mathcal{B}_{\mathcal{F}_1}}(\mu,\nu) = d_{\mathcal{B}_{\mathcal{F}_2}}(\mu,\nu) = 0$.
\end{itemize}
\end{proof}

\section{Preliminaries on Legendre Polynomials and spherical harmonics} \label{sec:preliminaries}
In the notation of \autoref{sec:sep_f1_f2_ipm}, $P_{k,d}(t)$ denotes the \textit{Legendre polynomial} of degree $k$ in dimension $d$.

It is known (equation (2.78), \cite{atkinson2012spherical}) that 
\begin{align} \label{eq:norm_legendre}
    \int_{-1}^1 P_{k,d}(t)^2 (1-t^2)^{\frac{d-3}{2}} dt = \frac{|\mathbb{S}^{d-1}|}{N_{k,d} |\mathbb{S}^{d-2}|} = \frac{\sqrt{\pi} \Gamma(\frac{d-1}{2})}{N_{k,d} \Gamma(\frac{d}{2})},
\end{align}
where (equation (2.10), \cite{atkinson2012spherical})
\begin{align} \label{eq:n_kd_def}
    N_{k,d} = \frac{(2k + d - 2)(k + d -3)!}{k! (d-2)!} 
\end{align}
is the dimension of the space of homogeneous harmonic polynomials of degree $k$ in $\R^d$.

Also, the following uniform bound holds (equation (2.116), \cite{atkinson2012spherical}):
\begin{align} \label{eq:uniform_bound}
    |P_{k,d}(t)| \leq 1, \quad k \geq 0, \ d \geq 2, \ t \in [-1,1].
\end{align}
The bound is achieved at $t=1, -1$.

There is also a crucial link between Legendre polynomials and their derivatives (equation (2.90), \cite{atkinson2012spherical}):
\begin{align} \label{eq:legendre_derivatives}
    P_{k,d}^{(j)}(t) = \frac{k! (k+j+d-3)! \Gamma(\frac{d-1}{2})}{2^j (k-j)! (k+d-3)! \Gamma(j + \frac{d-1}{2})} P_{k-j, d+2j}(t),
\end{align}
where $k \geq j$ and $d \geq 2$. Note that for $k < j$, $P_{k,d}^{(j)}(t) = 0$.

We recall two important equalities. Let $\{Y_{k,j} \ | \ 1 \leq j \leq N_{k,d}\}$ be an orthonormal basis of the space of homogeneous harmonic polynomials over $\R^d$ of degree $k$, with real coefficients (some works like \cite{atkinson2012spherical} consider complex coefficients and all the results are unchanged up to complex conjugates). That is, $\int_{\mathbb{S}^d} Y_{k,j}(x) Y_{k,i}(x) d\tau(x) = \delta_{ij}$, where $\tau$ is the uniform probability measure over $\mathbb{S}^{d-1}$. Then, the addition theorem (Thm. 4.11, \cite{efthimiou2014spherical}) states that
\begin{align}
    \sum_{j=1}^{N_{k,d}} Y_{k,j}(x) Y_{k,j}(y) = N_{k,d} P_{k,d}(\langle x, y \rangle) 
\end{align}
The Funk-Hecke formula (Thm 2.22, \cite{atkinson2012spherical}) states that when $\int_{-1}^{1} |f(t)| (1-t^2)^{\frac{d-3}{2}} dt < +\infty$, for any linear combination $Y_{k}$ of $\{ Y_{k,j} \ | \ 1 \leq j \leq N_{k,d} \}$ and for any $x \in \mathbb{S}^{d-1}$,
\begin{align}
    \int_{\mathbb{S}^{d-1}} f(\langle x, y \rangle) Y_{k}(y) d\tau(y) = \frac{|\mathbb{S}^{d-2}|}{|\mathbb{S}^{d-1}|} Y_k(x) \int_{-1}^{1} P_{k,d}(t) f(t) (1-t^2)^{\frac{d-3}{2}} \ dt.
\end{align}

\section{Proofs of \autoref{sec:sep_f1_f2_ipm}} \label{sec:proof_sep}

\propqdprob*
\begin{proof}
Clearly, $\frac{d\mu_d}{d\lambda}(x) \geq 0$ and $\frac{d\nu_d}{d\lambda}(x) \geq 0$ for all $x \in \mathbb{S}^{d-1}$. Let $C_{+} = \{x \in \mathbb{S}^{d-1} | L_{k,d}(x) > 0 \}$ and $C_{-} = \{x \in \mathbb{S}^{d-1} | L_{k,d}(x) \leq 0 \}$. For $\mu_d$ and $\nu_d$ to be probability measures, $\gamma_{k,d}$ must fulfill
\begin{align} \label{eq:gamma_cond}
    1 = \gamma_{k,d} \int_{C_{+}} \frac{L_{k,d}(x)}{|\mathbb{S}^{d-1}|}  \ d\lambda(x) \text{ and } 1 = - \gamma_{k,d} \int_{C_{-}} \frac{L_{k,d}(x)}{|\mathbb{S}^{d-1}|}  \ d\lambda(x).
\end{align}
By equation (1.17) of \cite{atkinson2012spherical}, if we parametrize $\mathbb{S}^{d-1}$ as $x = t e_d + (1-t^2)^{1/2} \xi_{(d-1)}$ with $t \in [-1,1]$ and $\xi_{(d-1)} \in \mathbb{S}^{d-2}$, we have 
\begin{align} \label{eq:hausdorff_change}
d\lambda(x) = (1-t^2)^{\frac{d-3}{2}} \ dt \ d\lambda_{(d-2)}(\xi_{(d-1)}),
\end{align}
where $\lambda_{(d-2)}$ denotes the Hausdorff measure of $\mathbb{S}^{d-2}$. Hence,
\begin{align}
\begin{split} \label{eq:int_zero}
    \int_{\mathbb{S}^{d-1}} \frac{L_{k,d}(x)}{|\mathbb{S}^{d-1}|} \ d\lambda(x) &= \int_{\mathbb{S}^{d-2}} \int_{-1}^{1} \frac{L_{k,d}(t e_d + (1-t^2)^{1/2} \xi_{(d-1)})}{|\mathbb{S}^{d-1}|} (1-t^2)^{\frac{d-3}{2}} dt \ d\lambda_{(d-2)}(\xi_{(d-1)}) \\ &= \frac{|\mathbb{S}^{d-2}|}{|\mathbb{S}^{d-1}|} \int_{-1}^{1} P_{k,d}(t) (1-t^2)^{\frac{d-3}{2}} dt = 0
\end{split}
\end{align}
The right-hand side is equal to zero because $P_{0,d} \equiv 1$, and the Legendre polynomials are orthogonal with respect to the scalar product with factor $(1-t^2)^{\frac{d-3}{2}}$ (\autoref{sec:sep_f1_f2_ipm}). Equation \eqref{eq:int_zero} implies that $\int_{C_{+}} \frac{L_{k,d}(x)}{|\mathbb{S}^{d-1}|}  \ d\lambda(x) = - \int_{C_{-}} \frac{L_{k,d}(x)}{|\mathbb{S}^{d-1}|}  \ d\lambda(x)$, which means that conditions \eqref{eq:gamma_cond} are feasible. We also have that \begin{align}
    \gamma_{k,d} = 2 \left(\int_{\mathbb{S}^{d-1}} \frac{|L_{k,d}(x)|}{|\mathbb{S}^{d-1}|} \ d\lambda(x)\right)^{-1} = 2 \left(\int_{\mathbb{S}^{d-1}} |L_{k,d}(x)| \ d\tau(x) \right)^{-1}. 
\end{align} 
\end{proof}

\begin{lemma} \label{lem:d_f2_comp}
For $\mu_d, \nu_d$ with densities given by \eqref{eq:densities_mu_nu}, 
we have
    \begin{align}
        d_{\mathcal{B}_{\mathcal{F}_2}}(\mu_d,\nu_d) = \gamma_{k,d} \frac{|\mathbb{S}^{d-2}|}{|\mathbb{S}^{d-1}|} \frac{1}{\sqrt{N_{k,d}}} \left| \int_{-1}^{1} P_{k,d}(t) \sigma(t) (1-t^2)^{\frac{d-3}{2}} \ dt \right|
    \end{align}
\end{lemma}
\begin{proof}
By \autoref{lem:f2_exp}, we have
\begin{align}
\begin{split} \label{eq:d_f2_development}
d^2_{\mathcal{B}_{\mathcal{F}_2}}(\mu_d,\nu_d) &= 
\int_{\mathbb{S}^{d-1}} \left( \int_{\mathbb{S}^{d-1}} \sigma(\langle x, \theta \rangle) d(\mu_d-\nu_d)(x) \right)^2 d\tau(\theta)
\\ &= \int_{\mathbb{S}^{d-1}} \left( \int_{\mathbb{S}^{d-1}} \sigma(\langle x, \theta \rangle) \frac{\gamma_{k,d}}{|\mathbb{S}^{d-1}|} L_{k,d}(x) \ d\lambda(x) \right)^2\ d\tau(\theta)
\\ &= \gamma_{k,d}^2 \int_{\mathbb{S}^{d-1}} \left( \int_{\mathbb{S}^{d-1}} \sigma(\langle x, \theta \rangle) L_{k,d}(x) \ d\tau(x) \right)^2\ d\tau(\theta)
\end{split}
\end{align}
Now, we reproduce the argument of \cite{bach2017breaking}. If we define $g(\theta) = \int_{\mathbb{S}^{d-1}} \sigma(\langle x, \theta \rangle) L_{k,d}(x) \ d\tau(x)$, since $L_{k,d}(x)$ is a homogeneous harmonic polynomial of degree $d$, by the Funk-Hecke formula we can write
\begin{align} \label{eq:g_lambda}
    g(\theta) = \frac{|\mathbb{S}^{d-2}|}{|\mathbb{S}^{d-1}|} L_{k,d}(\theta) \int_{-1}^{1} P_{k,d}(t) \sigma(t) (1-t^2)^{\frac{d-3}{2}} \ dt = \lambda_{k,d} L_{k,d}(\theta),
\end{align}
where $\lambda_{k,d} = \frac{|\mathbb{S}^{d-2}|}{|\mathbb{S}^{d-1}|} \int_{-1}^{1} P_{k,d}(t) \sigma(t) (1-t^2)^{\frac{d-3}{2}} \ dt$. Note as well that
\begin{align}
\begin{split} \label{eq:norm_prob}
    &\int_{\mathbb{S}^{d-1}} L_{k,d}(\theta)^2\ d\tau(\theta) \\ &= \frac{1}{|\mathbb{S}^{d-1}|} \int_{\mathbb{S}^{d-2}} \int_{-1}^{1} L_{k,d}(t e_d + (1-t^2)^{1/2} \xi_{(d-1)})^2\ (1-t^2)^{\frac{d-3}{2}} dt \ d\lambda_{(d-2)}(\xi_{(d-1)}) \\ &= \frac{|\mathbb{S}^{d-2}|}{|\mathbb{S}^{d-1}|} \int_{-1}^{1} P_{k,d}(t)^2\ (1-t^2)^{\frac{d-3}{2}} dt = \frac{1}{N_{k,d}}
\end{split}
\end{align}
In the first equality we used the same change of variables as in the proof of Lemma \ref{prop:qd_prob}, in the second equality we used that $P_{k,d}(t) = L_{k,d}(t e_d + (1-t^2)^{1/2} \xi_{(d-1)})$ by definition, and the third equality relies on equation \eqref{eq:norm_legendre}.
Using \eqref{eq:g_lambda} and \eqref{eq:norm_prob}, the right-hand side of \eqref{eq:d_f2_development} becomes:
\begin{align}
\begin{split} \label{eq:long_development}
    &\left(\gamma_{k,d} \lambda_{k,d} \right)^2 \int_{\mathbb{S}^{d-1}} L_{k,d}(\theta)^2 \ d\tau(\theta) 
    = \left(\gamma_{k,d} \lambda_{k,d} \right)^2 \frac{1}{N_{k,d}} \\ &= \gamma_{k,d}^2 \left(\frac{|\mathbb{S}^{d-2}|}{|\mathbb{S}^{d-1}|} \int_{-1}^{1} P_{k,d}(t) \sigma(t) (1-t^2)^{\frac{d-3}{2}} \ dt \right)^2 \frac{1}{N_{k,d}}
\end{split}
\end{align}
\end{proof}

\begin{lemma} \label{lem:d_f1_comp}
For $\mu_d, \nu_d$ with densities given by \eqref{eq:densities_mu_nu}, 
we have
\begin{align} \label{eq:lem_d_f1}
    d_{\mathcal{B}_{\mathcal{F}_1}}(\mu_d,\nu_d) = |\gamma_{k,d}| \frac{|\mathbb{S}^{d-2}|}{|\mathbb{S}^{d-1}|} \left|\int_{-1}^{1} P_{k,d}(t) \sigma(t) (1-t^2)^{\frac{d-3}{2}} \ dt \right|
\end{align}
\end{lemma}
\begin{proof}
Using \autoref{lem:f1_exp}, we have
\begin{align}
\begin{split}
    d_{\mathcal{B}_{\mathcal{F}_1}}(\mu_d,\nu_d) &= \sup_{\theta \in \mathbb{S}^d} \left| \int \sigma(\langle x, \theta \rangle) d(\mu_d-\nu_d)(x) \right| = \sup_{\theta \in \mathbb{S}^d} \left| \int \sigma(\langle x, \theta \rangle) \frac{\gamma_{k,d}}{|\mathbb{S}^{d-1}|} L_{k,d}(x) d\lambda(x) \right| \\ &= \gamma_{k,d} \sup_{\theta \in \mathbb{S}^d} \left| \int \sigma(\langle x, \theta \rangle) L_{k,d}(x) d\tau(x) \right| \\ &= \gamma_{k,d} |\lambda_{k,d}| \sup_{\theta \in \mathbb{S}^d} \left| L_{k,d}(\theta) \right| = |\gamma_{k,d}| |\lambda_{k,d}| \\ &= |\gamma_{k,d}| \frac{|\mathbb{S}^{d-2}|}{|\mathbb{S}^{d-1}|} \left|\int_{-1}^{1} P_{k,d}(t) \sigma(t) (1-t^2)^{\frac{d-3}{2}} \ dt \right|,
\end{split}
\end{align}
In the fourth equality we used \eqref{eq:g_lambda}, in the fifth equality we used $\sup_{\theta \in \mathbb{S}^d} \left| L_{k,d}(\theta) \right| = 1$ by equation \eqref{eq:uniform_bound}, and in the sixth equality we used the definition of $\lambda_{k,d}$.
\end{proof}

\thmsepipm*
\begin{proof}
Plugging the results of \autoref{lem:d_f2_comp} and \autoref{lem:d_f1_comp}, we obtain
\begin{align} \label{eq:distances_ratio}
    \frac{d_{\mathcal{B}_{\mathcal{F}_1}}(\mu_d,\nu_d)}{d_{\mathcal{B}_{\mathcal{F}_2}}(\mu_d,\nu_d)} = \frac{1}{\frac{1}{\sqrt{N_{k,d}}}} = \sqrt{N_{k,d}} = \sqrt{\frac{(2k + d - 2)(k + d -3)!}{k! (d-2)!}}
\end{align}
The last equality follows from \eqref{eq:n_kd_def}. Equation \eqref{eq:log_ratio} follows from Stirling's approximation, which states that $\log n! = n \log n - n + O(\log n)$ and $\log (\Gamma(x)) = x \log x - x + O(\log x)$.

All that is left is checking that \eqref{eq:d_f1_separation} holds. \autoref{prop:qd_prob} states that $\gamma_{k,d} = 2 \left(\int_{\mathbb{S}^{d-1}} |L_{k,d}(x)| \ d\tau(x) \right)^{-1} = 2 \left(\frac{|\mathbb{S}^{d-2}|}{|\mathbb{S}^{d-1}|} \int_{-1}^{1} |P_{k,d}(t)| (1-t^2)^{\frac{d-3}{2}} \ dt \right)^{-1}$ for $\mu_d$ and $\nu_d$ to be probability measures. Thus, \autoref{lem:d_f1_comp} implies that
\begin{align}
d_{\mathcal{B}_{\mathcal{F}_1}}(\mu_d,\nu_d) =  
\frac{2 \left|\int_{-1}^{1} P_{k,d}(t) \sigma(t) (1-t^2)^{\frac{d-3}{2}} \ dt \right|}{\int_{-1}^{1} |P_{k,d}(t)| (1-t^2)^{\frac{d-3}{2}} dt}.
\end{align}

\end{proof}

\section{Proofs of \autoref{sec:sep_sd}} \label{sec:proofs_sep_sd}

\begin{lemma} \label{lem:int_grad}
Let $\nabla$ denote the Riemannian gradient. We have that
\begin{align}
    \int_{\mathbb{S}^d} \|\nabla L_{k,d}(x)\|^2 d\tau(x) = k (k + d -2) \frac{1}{N_{k,d}}
\end{align}
\end{lemma}
\begin{proof}
Through integration by parts, we have that
\begin{align} \label{eq:int_parts}
    \int_{\mathbb{S}^d} \|\nabla L_{k,d}(x)\|^2 d\tau(x) = \int_{\mathbb{S}^d} L_{k,d}(x) (-\Delta) L_{k,d}(x) d\tau(x),
\end{align}
where $\Delta$ denotes the Laplace-Beltrami operator. Since the restriction of $L_{k,d}(x)$ to $\mathbb{S}^{d-1}$ is a $k$-spherical harmonic and spherical harmonics are eigenfunctions of the Laplace-Beltrami operator (equation (3.19) of \cite{atkinson2012spherical}), we have
\begin{align} \label{eq:laplacian_eigenval}
    - \Delta L_{k,d}(x) = k (k + d -2) L_{k,d}(x). 
\end{align}
Plugging this into \eqref{eq:int_parts} and using equality \eqref{eq:norm_legendre}, we obtain
\begin{align}
\begin{split}
    &\int_{\mathbb{S}^d} \|\nabla L_{k,d}(x)\|^2 d\tau(x) = k (k + d -2) \int_{\mathbb{S}^d} L_{k,d}(x)^2 d\tau(x) = k (k + d -2) \frac{1}{N_{k,d}}. 
\end{split}
\end{align}
\end{proof}

\begin{lemma} \label{lem:euclidean_riemannian}
Let $\hat{\nabla} L_{k,d}(x)$ be the Euclidean gradient of $L_{k,d} : \R^d \rightarrow \R$ and $\nabla L_{k,d}(x)$ be the Riemannian gradient of $L_{k,d} : \mathbb{S}^{d-1} \rightarrow \R$. Then, 
\begin{align}
    \hat{\nabla} L_{k,d}(x) = \nabla L_{k,d}(x) + k L_{k,d}(x) x.
\end{align}
\end{lemma}
\begin{proof}
By definition, for any $x \in \mathbb{S}^{d-1}$, $\nabla L_{k,d}(x)$ is the projection of $\hat{\nabla} L_{k,d}(x)$ to $\mathbb{S}^{d-1}$ to $T_x \mathbb{S}^{d-1}$. That is,
\begin{align}
    \nabla L_{k,d}(x) = \hat{\nabla} L_{k,d}(x) - \langle \hat{\nabla} L_{k,d}(x), x \rangle x = \hat{\nabla} L_{k,d}(x) - \frac{\partial}{\partial r} \hat{\nabla} L_{k,d}(rx)\bigg\rvert_{r=1} x = \hat{\nabla} L_{k,d}(x) - k L_{k,d}(x) x.
\end{align}
In the last equality we used that $\frac{\partial}{\partial r} \hat{\nabla} L_{k,d}(rx) = \frac{k}{r} L_{k,d}(rx)$, which holds because $L_{k,d}$ is a homogeneous polynomial of degree $k$.
\end{proof}

\begin{lemma} \label{lem:derivative_spherical}
Let $\hat{\nabla} L_{k,d}(x)$ be the Euclidean gradient of $L_{k,d} : \R^d \rightarrow \R$. Each component of $\hat{\nabla} L_{k,d}(x)$ is a $(k-1)$-th spherical harmonic when restricted to $\mathbb{S}^{d-1}$.
\end{lemma}
\begin{proof}
Spherical harmonics of degree $k$ in dimension $d$ can be characterized as the restrictions in $\mathbb{S}^{d-1}$ of homogeneous harmonic polynomials of degree $k$ in $\R^d$ \citep{atkinson2012spherical}, and harmonic functions $f : \R^d \rightarrow \R$ are those such that $\Delta f = \sum_{i=1}^d \partial_{ii} f = 0$. 
Notice that the $i$-th partial derivative of a homogeneous harmonic polynomial $p$ of degree $k$ is a homogeneous harmonic polynomial of degree $k-1$. That is because (i) the derivative of a homogeneous polynomial of degree $k$ is a homogeneous polynomial of degree $k-1$ and (ii) by commutation of partial derivatives, we have 
\begin{align}
    \Delta (\partial_i p) = \sum_{j=1}^{d+1} \partial_{jj} \partial_i p = \sum_{j=1}^{d+1} \partial_i \partial_{jj} p = \partial_i (\Delta p) = 0. 
\end{align}
Thus, the restriction of $\partial_i p$ to $\mathbb{S}^{d-1}$ is a $(k-1)$-th spherical harmonic.
\end{proof}

\begin{lemma} \label{lem:f1_sd_lower}
let $\sigma: \R \rightarrow \R$ be an $\alpha$-positive homogeneous activation function of the form \eqref{eq:alpha_positive}.
\begin{align} \label{eq:f1_sd_lower}
    \text{SD}_{\mathcal{B}_{\mathcal{F}_1^d}}(\mu_d,\nu_d) \geq 
    |a + (-1)^{k+1} b|\gamma_{k,d} \lambda_{k,d}^{(\alpha+1)} \frac{k(d+k-3)}{\alpha + 1},
\end{align}
where
\begin{align}
    \lambda_{k,d}^{(\alpha+1)} =  \frac{|\mathbb{S}^{d-2}|}{|\mathbb{S}^{d-1}|} \int_{-1}^{1} P_{k,d}(t) (t)_{+}^{\alpha+1} (1-t^2)^{\frac{d-3}{2}} \ dt
\end{align}
\end{lemma}
\begin{proof}
For simplicity, we begin by considering the case $\sigma(x) = (x)_{+}^{\alpha}$.
Remark that $\mu_d = \tau$, i.e., the uniform Borel probability measure over $\mathbb{S}^{d-1}$. We have
\begin{align}
\begin{split} \label{eq:sd_f1_development}
    \text{SD}_{\mathcal{B}_{\mathcal{F}_1^d}}(\mu_d,\nu_d) &= \sup_{h \in \mathcal{B}_{\mathcal{F}_1^d}} \mathbb{E}_{\mu_d} [\text{Tr}(\mathcal{A}_{\nu_d} h(x))] \\ &= \sup_{h \in \mathcal{B}_{\mathcal{F}_1^d}} \mathbb{E}_{\mu_d} [\text{Tr}(\mathcal{A}_{\nu_d} h(x)) - \text{Tr}(\mathcal{A}_{\mu_d} h(x))] \\ &= \sup_{h \in \mathcal{B}_{\mathcal{F}_1^d}} \mathbb{E}_{\mu_d} \left[ \left( \nabla \log \left(\frac{d\nu_d}{d\tau}(x) \right) - \nabla \log \left(\frac{d\mu_d}{d\tau}(x) \right) \right)^{\top} h(x) \right]
    \\ &= \gamma_{k,d} \sup_{h \in \mathcal{B}_{\mathcal{F}_1^d}} \mathbb{E}_{\mu_d} \left[ \nabla L_{k,d}(x)^{\top} h(x) \right]
    \\ &= \gamma_{k,d} \sup_{\substack{\|\mu_i\|_{\text{TV}} \leq 1 \\ \sum_i z_i^2 = 1}} \sum_{i=1}^{d} z_i \int_{\mathbb{S}^{d-1}} \nabla_i L_{k,d}(x) \int_{\mathbb{S}^{d-1}} (\langle \theta, x \rangle)_{+}^{\alpha} \ d\mu_i(\theta) \ d\tau(x)
    \\ &= \gamma_{k,d} \sup_{\substack{\theta^{(i)} \in \mathbb{S}^{d-1} \\ \sum_i z_i^2 = 1}} \sum_{i=1}^{d} z_i \left| \int_{\mathbb{S}^{d-1}} \nabla_i L_{k,d}(x) (\langle \theta^{(i)}, x \rangle)_{+}^{\alpha} \ d\tau(x) \right|
    \\ &= \gamma_{k,d} \sqrt{\sum_{i=1}^{d} \sup_{\theta^{(i)} \in \mathbb{S}^{d-1}}\left( \int_{\mathbb{S}^{d-1}} \nabla_i L_{k,d}(x) (\langle \theta^{(i)}, x \rangle)_{+}^{\alpha} \ d\tau(x) \right)^2}.
\end{split}
\end{align}
In the second equality, we have applied the Stein identity. The third equality relies on the definition of the Stein operator (equation \eqref{eq:stein_operator}). In the fourth equality, we used that $\mu_d$ is uniform and $\nu_d$ has density given by \eqref{eq:nu_density}. In the sixth equality we have used that for any function $f$ and domain $K$, the supremum of $\int_{K} f d\mu$ over signed measures with total variation norm bounded by 1 is equal to $\sup_{K} f$. In the seventh equality we have used the Cauchy-Schwarz inequality.
At this point, notice that by Lemma \ref{lem:euclidean_riemannian}
\begin{align}
\begin{split} \label{eq:int_euclidean_riemannian}
    \int_{\mathbb{S}^{d-1}} \nabla_i L_{k,d}(x) (\langle \theta, x \rangle)_{+}^{\alpha} \ d\tau(x) &= \left( \int_{\mathbb{S}^{d-1}} \nabla L_{k,d}(x) (\langle \theta, x \rangle)_{+}^{\alpha} \ d\tau(x) \right)_i \\ &= \left( \int_{\mathbb{S}^{d-1}} (\hat{\nabla} L_{k,d}(x)-k L_{k,d}(x) x) (\langle \theta, x \rangle)_{+}^{\alpha} \ d\tau(x) \right)_i
\end{split}
\end{align}
On the one hand, by the Funk-Hecke formula, since $\hat{\nabla} L_{k,d}(x)$ is a $(k-1)$-th spherical harmonic (\autoref{lem:derivative_spherical}), we have that for any $\theta \in \mathbb{S}^{d-1}$,
\begin{align}
\begin{split} \label{eq:fh_f1}
    \int_{\mathbb{S}^{d-1}} \hat{\nabla} L_{k,d}(x) (\langle \theta, x \rangle)_{+}^{\alpha} \ d\tau(x) &= \frac{|\mathbb{S}^{d-2}|}{|\mathbb{S}^{d-1}|} \hat{\nabla} L_{k,d}(\theta) \int_{-1}^{1} P_{k-1,d}(t) (t)_{+}^{\alpha} (1-t^2)^{\frac{d-3}{2}} \ dt \\ &= \lambda_{k-1,d}^{(\alpha)} \hat{\nabla} L_{k,d}(\theta) 
\end{split}
\end{align}
where $\lambda_{k-1,d}^{(\alpha)}$ is defined accordingly. 

On the other hand, since $\hat{\nabla}_{\theta} ((\langle \theta, x \rangle)_{+}^{\alpha+1}) = (\alpha+1) (\langle \theta, x \rangle)_{+}^{\alpha} x$ for $\theta \in \R^d$, we have that for any $\theta \in \R^d$,
\begin{align}
\begin{split} \label{eq:alpha_f1}
    &\int_{\mathbb{S}^{d-1}} L_{k,d}(x) x (\langle \theta, x \rangle)_{+}^{\alpha} \ d\tau(x) \\ &= \frac{1}{\alpha+1} \int_{\mathbb{S}^{d-1}} L_{k,d}(x) \hat{\nabla}_{\theta} ((\langle \theta, x \rangle)_{+}^{\alpha+1}) \ d\tau(x) \\ &= \frac{1}{\alpha+1} \hat{\nabla}_{\theta} \int_{\mathbb{S}^{d-1}} L_{k,d}(x) (\langle \theta, x \rangle)_{+}^{\alpha+1} \ d\tau(x) \\ &= \frac{1}{\alpha+1} \frac{|\mathbb{S}^{d-2}|}{|\mathbb{S}^{d-1}|} \hat{\nabla}_{\theta} (L_{k,d}(\theta) \|\theta\|^{\alpha+1-k}) \int_{-1}^{1} P_{k,d}(t) (t)_{+}^{\alpha+1} (1-t^2)^{\frac{d-3}{2}} \ dt \\ &= \frac{1}{\alpha+1} \lambda_{k,d}^{(\alpha+1)} \hat{\nabla}_{\theta} (L_{k,d}(\theta) \|\theta\|^{\alpha+1-k}),
\end{split}
\end{align}
In the third equality, we used the Funk-Hecke formula, which says that for any $\theta \in \mathbb{S}^{d-1}$, $\int_{\mathbb{S}^{d-1}} L_{k,d}(x) (\langle \theta, x \rangle)_{+}^{\alpha+1} \ d\tau(x) = L_{k,d}(\theta) \int_{-1}^{1} P_{k,d}(t) (t)_{+}^{\alpha+1} (1-t^2)^{\frac{d-3}{2}} \ dt$. To obtain the equality for a general $\theta \in \mathbb{R}^{d}$, we must use add the factor $\|\theta\|^{\alpha+1-k}$ so that the two sides have the same homogeneity parameter, yielding $\int_{\mathbb{S}^{d-1}} L_{k,d}(x) (\langle \theta, x \rangle)_{+}^{\alpha+1} \ d\tau(x) = L_{k,d}(\theta) \|\theta\|^{\alpha+1-k} \int_{-1}^{1} P_{k,d}(t) (t)_{+}^{\alpha+1} (1-t^2)^{\frac{d-3}{2}} \ dt$. 

And for $\theta \in \mathbb{S}^{d-1}$, we have
\begin{align}
\begin{split} \label{eq:alpha_1_k}
    \hat{\nabla}_{\theta} (L_{k,d}(\theta) \|\theta\|^{\alpha+1-k}) &= 
    \hat{\nabla}_{\theta} L_{k,d}(\theta) \|\theta\|^{\alpha+1-k} + (\alpha+1-k) L_{k,d}(\theta) \|\theta\|^{\alpha-1-k} \theta \\ &= \hat{\nabla}_{\theta} L_{k,d}(\theta) + (\alpha+1-k) L_{k,d}(\theta) \theta
\end{split}
\end{align}
Thus, the right-hand side of \eqref{eq:sd_f1_development} can be developed as
\begin{align}
\begin{split} \label{eq:sup_rewritten}
    \gamma_{k,d} \sqrt{\sum_{i=1}^{d} \sup_{\theta^{(i)}} \left( \left( \lambda_{k-1,d}^{(\alpha)} - \frac{k}{\alpha+1} \lambda_{k,d}^{(\alpha+1)} \right) \hat{\nabla}_i L_{k,d}(\theta^{(i)}) -\frac{k(\alpha+1-k)}{\alpha+1} \lambda_{k,d}^{(\alpha+1)} L_{k,d}(\theta^{(i)}) \theta^{(i)}_i \right)^2}   
\end{split}
\end{align}
\cite{bach2017breaking} (App. D.2) shows the following equality
\begin{align}
\begin{split} \label{eq:bach_equality}
    &\frac{|\mathbb{S}^{d-2}|}{|\mathbb{S}^{d-1}|} \int_{-1}^{1} P_{k,d}(t) (t)_{+}^{\alpha} (1-t^2)^{\frac{d-3}{2}} \ dt \\ &= 
    \begin{cases}
    0 &\text{if} \ k \equiv \alpha \ (\text{mod} \ 2), \ k > \alpha\\
    \frac{\Gamma(d/2)}{\sqrt{\pi}\Gamma((d-1)/2)} \frac{\alpha!(-1)^{(k-1-\alpha)/2}}{2^k} \frac{\Gamma\left(\frac{d-1}{2}\right) \Gamma\left( k-\alpha \right)}{\Gamma\left( \frac{k}{2} - \frac{\alpha}{2} + \frac{1}{2} \right)\Gamma\left( \frac{k}{2} + \frac{d}{2} + \frac{\alpha}{2} \right)} &\text{if} \ k \not\equiv \alpha \ (\text{mod} \ 2), \ k \geq \alpha+1
    \end{cases}
\end{split}
\end{align}
Notice that in \cite{bach2017breaking} the factor $\frac{d-1}{2 \pi}$ is a typo, and should instead be $\frac{\Gamma(d/2)}{\sqrt{\pi}\Gamma((d-1)/2)}$. Using equality \eqref{eq:bach_equality}, we get that when $k \not\equiv \alpha \ (\text{mod} \ 2)$ and $k \geq \alpha + 1$,
\begin{align}
\begin{split}
    \lambda_{k,d}^{(\alpha)} &= \frac{|\mathbb{S}^{d-2}|}{|\mathbb{S}^{d-1}|} \int_{-1}^{1} P_{k,d}(t) (t)_{+}^{\alpha} (1-t^2)^{\frac{d-3}{2}} \ dt \\ &= 
    \frac{\Gamma(d/2)}{\sqrt{\pi}\Gamma((d-1)/2)} \frac{\alpha! (-1)^{(k-\alpha-1)/2}}{2^k} \frac{\Gamma((d-1)/2) \Gamma(k-\alpha)}{\Gamma(\frac{k-\alpha+1}{2}) \Gamma(\frac{k+d+\alpha-1}{2})},
\end{split}
\end{align}
and $\lambda_{k,d}^{(\alpha)} = 0$ otherwise.
Thus,
\begin{align}
\begin{split}
    \frac{\lambda_{k-1,d}^{(\alpha)}}{\frac{k}{\alpha + 1}\lambda_{k,d}^{(\alpha+1)}} = \frac{\alpha + 1}{k} 
    \frac{\frac{\alpha!}{2^{k-1}}}{\frac{(\alpha+1)!}{2^{k}}} \frac{\frac{1}{\Gamma(\frac{k+d+\alpha-2}{2})}}{\frac{1}{\Gamma(\frac{k+d+\alpha}{2})}} 
    = \frac{k+d+\alpha-2}{k}
\end{split}
\end{align}
Hence, the arguments of the suprema in \eqref{eq:sup_rewritten} can be rewritten as
\begin{align} 
\begin{split} \label{eq:sup_rewritten2}
    &\frac{k}{\alpha+1} \lambda_{k,d}^{(\alpha+1)} \left( \left(\frac{k+d+\alpha-2}{k} - 1\right) \hat{\nabla}_i L_{k,d}(\theta^{(i)}) -(\alpha+1-k) L_{k,d}(\theta^{(i)}) \theta^{(i)}_i \right) 
    \\ &=  \frac{k}{\alpha+1} \lambda_{k,d}^{(\alpha+1)} \left( \frac{d+\alpha-2}{k} (\nabla_i L_{k,d}(\theta^{(i)}) + k L_{k,d}(\theta^{(i)}) \theta^{(i)}_i) -(\alpha+1-k) L_{k,d}(\theta^{(i)}) \theta^{(i)}_i \right)
\end{split}
\end{align}
If we substitute $i=d$ and $\theta^{(i)} = e_d$ in this expression, and use that $\nabla_d L_{k,d}(e_d) = 0$ (by the fact that $\nabla L_{k,d}(e_d) \in T_{e_d} \mathbb{S}^{d-1}$) and $L_{k,d}(e_d) = 1$ we obtain
\begin{align} \label{eq:dim_d_sd}
    \lambda_{k,d}^{(\alpha+1)} \left(\frac{(d+\alpha-2) k}{\alpha+1} 
    + \frac{k(k-\alpha-1)}{\alpha+1} \right) = \lambda_{k,d}^{(\alpha+1)} \frac{k(d+k-3)}{\alpha + 1},
\end{align}
which means that \eqref{eq:sup_rewritten} is lower-bounded by $\gamma_{k,d} \lambda_{k,d}^{(\alpha+1)} \frac{k(d+k-3)}{\alpha + 1}.$

When $\sigma(x) = (-x)_{+}^{\alpha}$, we reproduce the same argument. In this case, equation \eqref{eq:fh_f1} becomes $\int_{\mathbb{S}^{d-1}} \hat{\nabla} L_{k,d}(x) (- \langle \theta, x \rangle)_{+}^{\alpha} \ d\tau(x) = \lambda_{k-1,d}^{(\alpha)} \hat{\nabla} L_{k,d}(-\theta) = (-1)^{k+1} \lambda_{k-1,d}^{(\alpha)} \hat{\nabla} L_{k,d}(\theta)$, where we have used that $L_{k,d}(-\theta) = (-1)^k L_{k,d}(\theta)$. 
Since $\hat{\nabla}_{\theta} ((-\langle \theta, x \rangle)_{+}^{\alpha+1}) = -(\alpha+1) (-\langle \theta, x \rangle)_{+}^{\alpha} x$, equation \eqref{eq:alpha_f1} becomes $\int_{\mathbb{S}^{d-1}} L_{k,d}(x) x (-\langle \theta, x \rangle)_{+}^{\alpha} \ d\tau(x) = - \frac{1}{\alpha+1} \lambda_{k,d}^{(\alpha+1)} \hat{\nabla}_{\theta} (L_{k,d}(-\theta) \|-\theta\|^{\alpha+1-k})$. Since $L_{k,d}(-\theta) = (-1)^k L_{k,d}(\theta)$, we have that $(\hat{\nabla} L_{k,d})(-\theta) = (-1)^{k+1} \hat{\nabla} L_{k,d}(\theta)$.
Thus, equation \eqref{eq:alpha_1_k} becomes 
\begin{align}
\begin{split}
    -\hat{\nabla}_{\theta} (L_{k,d}(-\theta) \|-\theta\|^{\alpha+1-k}) &= 
    -\hat{\nabla}_{\theta} (L_{k,d}(-\theta)) \|-\theta\|^{\alpha+1-k}  -\hat{\nabla}_{\theta} (\|-\theta\|^{\alpha+1-k}) L_{k,d}(-\theta) \\ &= (\hat{\nabla}_{\theta} L_{k,d})(-\theta) \|-\theta\|^{\alpha+1-k} - (\alpha+1-k) \|-\theta\|^{\alpha-1-k} \theta L_{k,d}(-\theta) \\ &= (-1)^{k+1} \left( \hat{\nabla} L_{k,d}(\theta) + (\alpha+1-k) \theta L_{k,d}(\theta) \right)
\end{split}
\end{align}
Hence, for $\sigma(x) = (-x)_{+}^{\alpha}$ the expression \eqref{eq:sup_rewritten} is unchanged, and the rest of the argument holds in the same way. When $\sigma(x) = a (x)_{+}^{\alpha} + b (-x)_{+}^{\alpha}$, the argument of the square root in expression \eqref{eq:sup_rewritten} gets multiplied by $|a + (-1)^{k+1} b|$, and this factor is carried over for the rest of the argument. This concludes the proof.
\end{proof}

\begin{lemma} \label{lem:sd_f2_upper}
Let $\sigma: \R \rightarrow \R$ be an $\alpha$-positive homogeneous activation function of the form \eqref{eq:alpha_positive}. Then, $\text{SD}_{\mathcal{B}_{\mathcal{F}_2^d}}(\mu_d,\nu_d)$ is upper-bounded by
\begin{align} 
    |a+(-1)^{k+1} b| \gamma_{k,d} \lambda_{k,d}^{(\alpha+1)} \sqrt{\frac{2}{N_{k,d}} \left(k (k + d -2) \left( \frac{d+\alpha-2}{\alpha+1} \right)^2 + \left( \frac{k(d+k-3)}{\alpha + 1} \right)^2 \right) }.
\end{align}
\end{lemma}
\begin{proof} 
For simplicity, we begin by considering the case $\sigma(x) = (x)_{+}^{\alpha}$.
\begin{align}
\begin{split} \label{eq:sd_f2_development}
    \text{SD}_{\mathcal{B}_{\mathcal{F}_2^d}}(\mu_d,\nu_d) &= \sup_{h \in \mathcal{B}_{\mathcal{F}_2^d}} \mathbb{E}_{\mu_d} [\text{Tr}(\mathcal{A}_{\nu_d} h(x))] \\ &= \sup_{h \in \mathcal{B}_{\mathcal{F}_2^d}} \mathbb{E}_{\mu_d} [\text{Tr}(\mathcal{A}_{\nu_d} h(x)) - \text{Tr}(\mathcal{A}_{\mu_d} h(x))] \\ &= \sup_{h \in \mathcal{B}_{\mathcal{F}_2^d}} \mathbb{E}_{\mu_d} \left[ \left( \nabla \log \left(\frac{d\nu_d}{d\tau}(x) \right) - \nabla \log \left(\frac{d\mu_d}{d\tau}(x) \right) \right)^{\top} h(x) \right]
    \\ &= \gamma_{k,d} \sup_{h \in \mathcal{B}_{\mathcal{F}_2^d}} \mathbb{E}_{\mu_d} \left[ \nabla L_{k,d}(x)^{\top} h(x) \right]
    = \gamma_{k,d} \sup_{\substack{\|h_i\|_{\mathcal{F}_2} \leq 1 \\ \sum_i z_i^2 = 1}} z_i \mathbb{E}_{\mu_d} \left[ \nabla_i L_{k,d}(x) \langle k(x,\cdot), h_i \rangle_{\mathcal{F}_2} \right]
    \\ &= \gamma_{k,d} \sup_{\substack{\|h_i\|_{\mathcal{F}_2} \leq 1 \\ \sum_i z_i^2 = 1}} z_i \left\langle \mathbb{E}_{\mu_d} \left[ \nabla_i L_{k,d}(x) k(x,\cdot) \right], h_i \right\rangle_{\mathcal{F}_2}
    \\ &= \gamma_{k,d} \sqrt{ \sum_{i=1}^{d} \left\| \mathbb{E}_{\mu_d} \left[ \nabla_i L_{k,d}(x) k(x,\cdot) \right] \right\|_{\mathcal{F}_2}^2}. 
\end{split}
\end{align}
And we can rewrite the right-hand side as
\begin{align}
\begin{split} \label{eq:rh_1}
    &\gamma_{k,d} \sqrt{ \sum_{i=1}^{d} \iint_{\mathbb{S}^{d-1} \times \mathbb{S}^{d-1}} \nabla_i L_{k,d}(x) k(x,y) \nabla_i L_{k,d}(y) d\tau(x) d\tau(y)} \\ &= \gamma_{k,d} \sqrt{ \sum_{i=1}^{d} \int_{\mathbb{S}^{d-1}} \left(\int_{\mathbb{S}^{d-1}} \nabla_i L_{k,d}(x) (\langle x, \theta \rangle)_{+}^{\alpha} d\tau(x) \right)^2 d\tau(\theta)}.   
\end{split}
\end{align}
At this point, we express $\int_{\mathbb{S}^{d-1}} \nabla_i L_{k,d}(x) (\langle \theta, x \rangle)^{\alpha}_{+} \ d\tau(x)$ using the development in equations \eqref{eq:int_euclidean_riemannian}, \eqref{eq:fh_f1}, \eqref{eq:alpha_f1}, \eqref{eq:alpha_1_k}:
\begin{align}
    \begin{split}
        &\left(\int_{\mathbb{S}^{d-1}} \nabla_i L_{k,d}(x) (\langle \theta_i, x \rangle)^{\alpha}_{+} \ d\tau(x) \right)^2 \\
        &= \left(\frac{k}{\alpha+1} \lambda_{k,d}^{(\alpha+1)}\right)^2 \left( \left(\frac{k+d+\alpha-2}{k} - 1\right) \hat{\nabla}_i L_{k,d}(\theta) -(\alpha+1-k) L_{k,d}(\theta) \theta_i \right)^2 \\
        &= \left(\frac{k}{\alpha+1} \lambda_{k,d}^{(\alpha+1)} \right)^2 \left( \frac{d+\alpha-2}{k} (\nabla_i L_{k,d}(\theta) + k L_{k,d}(\theta) \theta_i) -(\alpha+1-k) L_{k,d}(\theta) \theta_i \right)^2
        \\ &\leq 2 A^2 (\nabla_i L_{k,d}(\theta))^2 + 2 B^2 \left( L_{k,d}(\theta) \theta_i \right)^2,
    \end{split}
\end{align}
where, using the computations in the proof of Lemma \ref{lem:f1_sd_lower}, 
\begin{align}
\begin{split}
    A &= 
    \lambda_{k,d}^{(\alpha+1)} \frac{d+\alpha-2}{\alpha+1},
    \\ B &= \lambda_{k,d}^{(\alpha+1)} \left(\frac{(d+\alpha-2) k}{\alpha+1} 
    + \frac{k(k-\alpha-1)}{\alpha+1} \right) = \lambda_{k,d}^{(\alpha+1)} \frac{k(d+k-3)}{\alpha + 1}. 
\end{split}    
\end{align}
Thus, the right-hand side of \eqref{eq:rh_1} is upper-bounded by:
\begin{align} \label{eq:rh_2}
    \gamma_{k,d} \sqrt{ 2 A^2 \int_{\mathbb{S}^{d-1}} \|\nabla L_{k,d}(\theta)\|^2 d\tau(\theta) + 2B^2 \int_{\mathbb{S}^{d-1}} L_{k,d}(\theta)^2 d\tau(\theta)}
\end{align}
We can use Lemma \ref{lem:int_grad} to compute the first integral: $\int_{\mathbb{S}^d} \|\nabla L_{k,d}(x)\|^2 d\tau(x) = k (k + d -2) \frac{1}{N_{k,d}}$.
And for the second integral we have $\int_{\mathbb{S}^{d-1}} L_{k,d}(\theta)^2 d\tau(\theta) = \frac{1}{N_{k,d}}$
by equation \eqref{eq:norm_prob}. Substituting everything into \eqref{eq:rh_2} yields
\begin{align} \label{eq:upper_bound_relu}
    \gamma_{k,d} \lambda_{k,d}^{(\alpha+1)} \sqrt{\frac{2}{N_{k,d}} \left(k (k + d -2) \left( \frac{d+\alpha-2}{\alpha+1} \right)^2 + \left( \frac{k(d+k-3)}{\alpha + 1} \right)^2 \right) }.
\end{align}
For the general case $\sigma(x) = a (x)_{+}^{\alpha} + b (-x)_{+}^{\alpha}$, we use arguments analogous to those of \autoref{lem:f1_sd_lower}, and we obtain that the upper-bound \eqref{eq:upper_bound_relu} gets multiplied by a factor $|a + (-1)^{k+1} b|$.
\end{proof}

\thmseparationsd*
\begin{proof}
We obtain \eqref{eq:sd_sep} from \autoref{lem:f1_sd_lower} and \autoref{lem:sd_f2_upper}. Taking the logarithm and using Stirling's approximation yields \eqref{eq:sd_sep_log}. The only relevant factor is $\log(\sqrt{N_{k,d}})$, as the other ones are $O(\log(k+d))$.
\end{proof}

\section{Proofs of \autoref{sec:bounds_f1_f2}} \label{sec:proofs_bounds}

\begin{lemma} [Approximation of Lipschitz-continuous functions on the unit ball by $\mathcal{F}_2$ functions, \cite{bach2017breaking} ] \label{lem:lipschitz_bach}
Let $\sigma(x) = (x)_{+}^{\alpha}$ be the $\alpha$-th power of the ReLu activation function, where $\alpha$ is a non-negative integer. For $\delta$ greater than a constant depending only on $d$, 
for any function $f : \mathbb{R}^d \rightarrow \R$ such that for all $x, y$ such that for any $\|x\|_q \leq R, \ \|y\|_q \leq R$ we have $|f(x)| \leq \eta$ and $|f(x) - f(y)| \leq \eta R^{-1} \|x-y\|_q$, 
there exists $h \in \mathcal{F}_2 (\R^d \times \{R\})$, such that $\|h\|_{\mathcal{F}_2} \leq \delta$ and
\begin{align}
\sup_{\|x\|_q \leq R} |h(x) - f(x)| \leq C(d,\alpha) \eta \left( \frac{R \delta}{\eta} \right)^{-\frac{1}{\alpha+(d-1)/2}} \log \left(\frac{R \delta}{\eta} \right)
\end{align}
\end{lemma}
\begin{proof}
See Proposition 6 of \cite{bach2017breaking}. Notice that in \cite{bach2017breaking} the factor in the bound is $\left( \frac{\delta}{\eta} \right)^{-2/(d+1)} \log \left(\frac{\delta}{\eta} \right)$, while we have $\left( \frac{R \delta}{\eta} \right)^{-2/(d+1)} \log \left(\frac{R \delta}{\eta} \right)$. 
The $R$ factor stems from the fact that we consider the neural network features to lie in $\mathbb{S}^d$, while \citet{bach2017breaking} considers them in the hypersphere of radius $R^{-1}$.
\end{proof}

\thmdfonespikedw*
\begin{proof}
We begin with the lower bound.
Let $U^{\star}$ and $f^{\star}$ be the matrix in $\mathcal{V}_k$ and function in $\text{Lip}_1(\R^k)$ where $\overline{\mathcal{W}}_{1,k}(\mu,\nu)$ is achieved (it is known in optimal transport that the supremum is achieved \citep{nilesweed2019estimation}), i.e.
\begin{align}
    \overline{\mathcal{W}}_{1,k}(\mu,\nu) = \mathbb{E}_{x \sim \mu} [f^{\star}(U^{\star} x)] - \mathbb{E}_{x \sim \nu} [f^{\star}(U^{\star} x)].
\end{align}
If $\mu$ (resp. $\nu$) is supported in the closed unit ball in $\R^d$, $\forall x \in \text{supp}(\mu), \ \|U^{\star} x\|_2 \leq \|x\|_2 \leq 1$. Thus, all that matters is the restriction of $f^{\star}$ to the closed unit ball of $\R^{k}$. We can apply \autoref{lem:lipschitz_bach} with $R=1, \eta=1$, which yields the existence of $h \in \mathcal{F}_2(\R^{k})$, such that $\|h\|_{\mathcal{F}_2} \leq \delta$ and
\begin{align}
\sup_{x \in \R^{k}, \|x\|_2 \leq 1} |h(x) - f^{\star}(x)| \leq C(k,\alpha) \delta^{-\frac{1}{\alpha+(k-1)/2}} \log \left(\delta \right),
\end{align}
when $\delta$ is larger than a constant depending on $k$ and $\alpha$. Thus,
\begin{align}
\sup_{x \in \R^{d}, \|x\|_2 \leq 1} |h(U^{\star} x) - f^{\star}(U^{\star} x)| \leq C(k,\alpha) \delta^{-\frac{1}{\alpha+(k-1)/2}} \log \left(\delta \right),
\end{align}
which implies that
\begin{align}
    |\overline{\mathcal{W}}_{1,k}(\mu,\nu) - (\mathbb{E}_{x \sim \mu} [h(U^{\star} x)] - \mathbb{E}_{x \sim \nu} [h(U^{\star} x)])| \leq 2 C(k,\alpha) \delta^{-\frac{1}{\alpha+(k-1)/2}} \log \left(\delta \right).
\end{align}
Now, $h \circ U^{\star}$ belongs to $\mathcal{F}_1(\R^d)$ by the argument of Section 4.6 of \cite{bach2017breaking}. Namely, if $h(x) = \int_{\mathbb{S}^{k}} (\langle \theta, (x,1) \rangle)^{\alpha}_{+} \ d\mu_h(\theta)$, we can write
\begin{align}
\begin{split}
    h(U^{\star} x) &= \int_{\mathbb{S}^{k}} (\langle \theta, (U^{\star} x,1) \rangle)^{\alpha}_{+} \ d\mu_h(\theta) = \int_{\mathbb{S}^{k}} (\langle \theta_{1:k}, U^{\star} x \rangle + \theta_{k+1})^{\alpha}_{+} \ d\mu_h(\theta) \\ &= \int_{\mathbb{S}^{k}} (\langle (U^{\star})^{\top} \theta_{1:k}, x \rangle + \theta_{k+1})^{\alpha}_{+} \ d\mu_h(\theta) = \int_{\mathbb{S}^{k}} (\langle ((U^{\star})^{\top}\theta_{1:k}, \theta_{k+1}), (x,1) \rangle)^{\alpha}_{+} \ d\mu_h(\theta) \\ &= \int_{\mathbb{S}^{k}} (\langle \theta, (x,1) \rangle)^{\alpha}_{+} \ d\tilde{\mu}_h(\theta),
\end{split}
\end{align}
where $\tilde{\mu}_h$ is the pushforward of $\mu_h$ by the map $\theta \mapsto ((U^{\star})^{\top}\theta_{1:k}, \theta_{k+1})$. The last equality follows from the fact that $\|((U^{\star})^{\top}\theta_{1:k}, \theta_{k+1})\|_2^2 = \theta_{1:k}^\top U^{\star} (U^{\star})^{\top} \theta_{1:k} + \theta_{k+1}^2 = \|\theta\|_2^2 = 1$. Moreover, this argument also shows that $h \circ U^{\star}$ has $\mathcal{F}_1$ norm $\gamma_1(h \circ U^{\star}) \leq \gamma_2(h) \leq \delta$. Hence, $\forall \mu, \nu \in \mathcal{P}(B_1(\R^d)),$ for $\delta$ larger than a constant depending on $k$,
\begin{align}
    \delta d_{\mathcal{B}_{\mathcal{F}_1}}(\mu,\nu) \geq \overline{\mathcal{W}}_{1,k}(\mu,\nu) - 2 C(k,\alpha) \delta^{-\frac{1}{\alpha+(k-1)/2}} \log \left(\delta \right).
\end{align}

The upper bound $\overline{\mathcal{W}}_{1,k}(\mu,\nu) \geq d_{\mathcal{B}_{\mathcal{F}_1}}(\mu,\nu)$ follows from 
\begin{align}
\begin{split}
    \overline{\mathcal{W}}_{1,k}(\mu,\nu) &= \max_{U \in \mathcal{V}_k} \sup_{f \in \text{Lip}_1(\R^k)}\mathbb{E}_{x \sim \mu} [f(U x)] - \mathbb{E}_{x \sim \nu} [f(U x)] \\ &\geq \max_{U \in \mathcal{V}_k} \sup_{f \in \mathcal{B}_{\mathcal{F}_1(\R^k)}}\mathbb{E}_{x \sim \mu} [f(U x)] - \mathbb{E}_{x \sim \nu} [f(U x)] \\ &=
    \sup_{f \in \mathcal{B}_{\mathcal{F}_1(\R^k)}} \mathbb{E}_{x \sim \mu} [f(x)] - \mathbb{E}_{x \sim \nu} [f(x)] = d_{\mathcal{B}_{\mathcal{F}_1}}(\mu,\nu)
\end{split}
\end{align}
In the second to last inequality we used once again that for all $f \in \mathcal{F}_1(\R^k)$ such that $\|f\|_{\mathcal{F}_1(\R^k)} = 1$, we have $f \circ U \in \mathcal{F}_1(\R^k)$ and $\|f \circ U\|_{\mathcal{F}_1(\R^k)} = 1$.
\end{proof}

\proptildetau*

\begin{proof}
The measure $d\tilde{\tau}(\sqrt{1-t^2} \xi, t) = \frac{1}{\pi} (1-t^2)^{-1/2} \ dt \ d\tau_{(d-1)}(\xi)$ is normalized because $\int_{\mathbb{S}^{d-1}} \int_{-1}^{1} (1-t^2)^{-1/2} \ dt \ d\tau_{(d-1)}(\xi) = \int_{-1}^{1} (1-t^2)^{-1/2} \ dt = \arcsin(1) - \arcsin(-1) = \pi$, where we used that $\int_{\mathbb{S}^{d-1}} d\tau_{(d-1)}(\xi) = 1$ by definition of $\tau_{(d-1)}$.
The characterization of the uniform measure $\tau$ follows from equation (1.17) of \cite{atkinson2012spherical}: $d\tau(\theta) = \frac{|\mathbb{S}^{d-1}|}{|\mathbb{S}^{d}|} (1-t^2)^{\frac{d-1}{2}} \ dt \ d\tau_{(d-1)}(\xi) = \frac{\Gamma((d+1)/2)}{\sqrt{\pi} \Gamma(d/2)} (1-t^2)^{\frac{d-1}{2}} \ dt \ d\tau_{(d-1)}(\xi)$.
For clarity, if we plug this change of variables into equation \eqref{eq:f2_kernel}, we obtain that the $F_2$ kernel reads:
\begin{align}
\begin{split}
    &k(x,y) = \int_{\mathbb{S}^d} \sigma(\langle (x,1), \theta \rangle) \sigma(\langle (y,1), \theta \rangle) d\tau(\theta) = \frac{\Gamma((d+1)/2)}{\sqrt{\pi} \Gamma(d/2)} \cdot \\ &\int_{\mathbb{S}^{d-1}} \int_{-1}^{1} \sigma\left(\langle (x,1), (\sqrt{1-t^2} \xi_{(d)}, t) \rangle \right) \sigma\left(\langle (y,1), (\sqrt{1-t^2} \xi, t) \rangle \right) (1-t^2)^{\frac{d-1}{2}} \ dt \ d\tau_{(d-1)}(\xi).
\end{split}    
\end{align}
\text
Notice that beyond the normalization factors, the main difference between $\tilde{k}$ and $k$ is the factor $(1-t^2)^{-1/2}$ instead of $(1-t^2)^{\frac{d-1}{2}}$.
\end{proof}

\begin{lemma} \label{lem:tilde_f2_bound}
Let $d_{\mathcal{B}_{\tilde{\mathcal{F}}_2}}$ be as defined in \eqref{eq:def_tilde_f2} and let $K = \{x \in \R^d | \|x\|_2 \leq 1 \} \times \{1\}$. Then, for any $\mu,\nu \in \mathcal{P}(K)$, $d^2_{\mathcal{B}_{\tilde{\mathcal{F}}_2}}(\mu,\nu)$ is lower-bounded by
\begin{align}
\begin{split}
\frac{1}{2 \pi} \frac{5}{6 \alpha 2^{\alpha/2}} \int_{\mathbb{S}^{d-1}} \sup_{\gamma \in [0,2\pi]} \left|\int_K \left(\langle (x,1), (\cos(\gamma) \xi_{(d)}, \sin(\gamma)) \rangle \right)^{\alpha}_{+} d(\mu-\nu)(x) \right|^3 \ d\tau_{(d-1)}(\xi_{(d)}).
\end{split}    
\end{align}
\end{lemma}
\begin{proof}
Using the change of variables $t = \sin(\gamma)$, we have
\begin{align}
\begin{split} \label{eq:d_tilde_f2_change}
    &d^2_{\mathcal{B}_{\tilde{\mathcal{F}}_2}}(\mu,\nu) \\ &= \frac{1}{\pi} \int_{\mathbb{S}^{d-1}} \int_{-1}^{1} \left( \int_K \left(\langle (x,1), (\sqrt{1-t^2} \xi_{(d)}, t) \rangle \right)^{\alpha}_{+} d(\mu-\nu)(x,1) \right)^2 (1-t^2)^{-1/2} \ dt \ d\tau_{(d-1)}(\xi_{(d)}) \\ &= \frac{1}{\pi} \int_{\mathbb{S}^{d-1}} \int_{-\pi/2}^{\pi/2} \left( \int_K \left(\langle (x,1), (\cos(\gamma) \xi_{(d)}, \sin(\gamma)) \rangle \right)^{\alpha}_{+} d(\mu-\nu)(x,1) \right)^2 \ d\gamma \ d\tau_{(d-1)}(\xi_{(d)}) \\ &= \frac{1}{2 \pi} \int_{\mathbb{S}^{d-1}} \int_{0}^{2\pi} \left( \int \left(\langle (x,1), (\cos(\gamma) \xi_{(d)}, \sin(\gamma)) \rangle \right)^{\alpha}_{+} d(\mu-\nu)(x,1) \right)^2 \ d\gamma \ d\tau_{(d-1)}(\xi_{(d)}).
\end{split}
\end{align}
We want to compute the Lipschitz constant of $\gamma \mapsto \int \left(\langle (x,1), (\cos(\gamma) \xi_{(d)}, \sin(\gamma)) \rangle \right)^{\alpha} d(\mu-\nu)(x)$. For $\alpha \geq 1$, the derivative of this mapping is:
\begin{align}
\begin{split}
    \int \alpha \ \left(\langle (x,1), (\cos(\gamma) \xi_{(d)}, \sin(\gamma)) \rangle \right)^{\alpha-1}_{+} \langle (x,1), (-\sin(\gamma) \xi_{(d)}, \cos(\gamma)) \rangle d(\mu-\nu)(x,1),
\end{split}
\end{align}
and its absolute value is upper-bounded by
\begin{align}
\begin{split}
    &2 \alpha |\langle (x,1), (\cos(\gamma) \xi_{(d)}, \sin(\gamma)) \rangle|^{\alpha-1} |\langle (x,1), (-\sin(\gamma) \xi_{(d)}, \cos(\gamma)) \rangle| \leq 2 \alpha \|(x,1)\|_2^{\alpha-1} \|(x,1)\|_2 \\ &= 2 \alpha \|(x,1)\|_2^{\alpha} \leq \alpha 2^{\alpha/2},
\end{split}
\end{align}
where we used that $\|x\|_2 \leq 1$ for $x$ in the support of $\mu$ or $\nu$. Thus, if we denote $s = \sup_{\gamma \in [0,2\pi]} \left|\int \left(\langle (x,1), (\cos(\gamma) \xi_{(d)}, \sin(\gamma)) \rangle \right)^{\alpha}_{+} d(\mu-\nu)(x) \right|$, we have
\begin{align}
\begin{split}
    &\int_{0}^{2\pi} \left( \int \left(\langle (x,1), (\cos(\gamma) \xi_{(d)}, \sin(\gamma)) \rangle \right)^{\alpha}_{+} d(\mu-\nu)(x) \right)^2 \ d\gamma \geq \int_{0}^{\frac{s}{\alpha 2^{\alpha/2}}} \left(s - \gamma \alpha 2^{\alpha/2} \right)^2 d\gamma \\ &= \int_{0}^{\frac{s}{\alpha 2^{\alpha/2}}} \left( s^2 - \gamma \alpha 2^{\alpha/2} s + \left( \gamma \alpha 2^{\alpha/2} \right)^2 \right) d\gamma = \frac{s^3}{\alpha 2^{\alpha/2}} - \frac{\alpha 2^{\alpha/2} s}{2} \left( \frac{s}{\alpha 2^{\alpha/2}} \right)^2 + \frac{\left( \alpha 2^{\alpha/2} \right)^2}{3} \left( \frac{s}{\alpha 2^{\alpha/2}} \right)^3 \\ &= \frac{5 s^3}{6 \alpha 2^{\alpha/2}}.
\end{split}
\end{align}
Hence,
\begin{align}
\begin{split}
    d^2_{\mathcal{B}_{\tilde{\mathcal{F}}_2}}(\mu,\nu) \geq \frac{1}{2 \pi} \frac{5}{6 \alpha 2^{\alpha/2}} \int_{\mathbb{S}^{d-1}} \sup_{\gamma \in [0,2\pi]} \left|\int \left(\langle (x,1), (\cos(\gamma) \xi_{(d)}, \sin(\gamma)) \rangle \right)^{\alpha}_{+} d(\mu-\nu)(x) \right|^3 \ d\tau_{(d-1)}(\xi_{(d)}).
\end{split}    
\end{align}
\end{proof}

\thmdtildeftwo*
\begin{proof}
We begin with the lower bound \eqref{eq:tilde_f2_lower}.
By the definition of the integral $1$-dimensional projection robust Wasserstein distance and the fact that the Stiefel manifold for $k=1$ is $\mathcal{S}^{d-1}$, we have
\begin{align}
    \underline{\mathcal{W}}_{1,1}(\mu,\nu) = \int_{\mathbb{S}^{d-1}} \mathcal{W}_1(u_{\#}\mu, u_{\#}\nu) d\tau(u),
\end{align}
where $u_{\#}\mu$ denotes the pushforward of $\mu$ by the map $\theta \mapsto \langle u, \theta \rangle$ and thus, $\mathcal{W}_1(u_{\#}\mu, u_{\#}\nu) = \min_{\pi \in \Gamma(\mu, \nu)} \int \|Ux-Uy\| d\pi(x,y)$. By the dual characterization of the 1-Wasserstein distance, for any $u \in \mathbb{S}^{d-1}$ we can write
\begin{align}
    \mathcal{W}_1(u_{\#}\mu, u_{\#}\nu) = \mathbb{E}_{x \sim \mu} [f^{\star}_u (\langle u, x \rangle)] - \mathbb{E}_{x \sim \nu} [f^{\star}_u(\langle u, x \rangle)]
\end{align}
for some function in $\text{Lip}_1(\R)$. Using the same argument as in \autoref{thm:d_f1_spiked_w}, \autoref{lem:lipschitz_bach} with $R=1, \eta=1$ yields the existence of $h_u \in \mathcal{F}_2(\R)$ such that $\|h_u\|_{\mathcal{F}_2} \leq \delta$ and
\begin{align}
\sup_{x \in \R, |x| \leq 1} |h_u(x) - f^{\star}_u(x)| \leq C(1,\alpha) \delta^{-\frac{1}{\alpha}} \log \left(\delta \right),
\end{align}
when $\delta$ is larger than a constant depending on $k$ and $\alpha$. Thus,
\begin{align}
\sup_{x \in \R^{d}, \|x\|_2 \leq 1} |h_u(\langle u, x \rangle) - f^{\star}_u(\langle u, x \rangle)| \leq C(1,\alpha) \delta^{-\frac{1}{\alpha}} \log \left(\delta \right),
\end{align}
which implies that
\begin{align}
\begin{split} \label{eq:wasserstein_approx}
    &|\mathcal{W}_1(u_{\#}\mu, u_{\#}\nu) - (\mathbb{E}_{x \sim \mu} [h_u(\langle u, x \rangle)] - \mathbb{E}_{x \sim \nu} [h_u(\langle u, x \rangle)])| \leq 2 C(1,\alpha) \delta^{-\frac{1}{\alpha}} \log \left(\delta \right) \\
    &\implies \mathbb{E}_{x \sim \mu} [h_u(\langle u, x \rangle)] - \mathbb{E}_{x \sim \nu} [h_u(\langle u, x \rangle)] \geq \mathcal{W}_1(u_{\#}\mu, u_{\#}\nu) - 2 C(1,\alpha) \delta^{-\frac{1}{\alpha}} \log \left(\delta \right).
\end{split}
\end{align}
And since $h_u(y) = \int_{\mathbb{S}^1} \sigma(\langle \theta, (y,1) \rangle) d\mu_{h_u}(\theta)$ for some $\mu_{h_u} \in \mathcal{M}(\mathbb{S}^1)$ such that $\|\mu_{h_u}\|_{\text{TV}} \leq \delta$, we have
\begin{align}
\begin{split} \label{eq:sup_bound}
    \mathbb{E}_{x \sim \mu} [h_u(\langle u, x \rangle)] - \mathbb{E}_{x \sim \nu} [h_u(\langle u, x \rangle)] &= \int \int_{\mathbb{S}^1} (\langle \theta, (\langle u, x \rangle,1) \rangle)^{\alpha}_{+} \ d\mu_{h_u}(\theta) \ d(\mu-\nu)(x) \\ &= \int_{\mathbb{S}^1} \int (\langle (\theta_1 u, \theta_2), (x,1) \rangle)^{\alpha}_{+} \ d(\mu-\nu)(x) \ d\mu_{h_u}(\theta) \\ &\leq \delta \sup_{\theta \in \mathbb{S}^1} \left|\int (\langle (\theta_1 u, \theta_2), (x,1) \rangle)^{\alpha}_{+} \ d(\mu-\nu)(x) \right|
    \\ &= \delta \sup_{\gamma \in [0,2\pi]} \left|\int (\langle (\cos(\gamma) u, \sin(\gamma)), (x,1) \rangle)^{\alpha}_{+} \ d(\mu-\nu)(x) \right|.
\end{split}
\end{align}
Hence, \eqref{eq:wasserstein_approx} and \eqref{eq:sup_bound} yield
\begin{align}
\begin{split} \label{eq:sup_gamma}
    &\delta \int_{\mathbb{S}^{d-1}} \sup_{\gamma \in [0,2\pi]} \left|\int (\langle (\cos(\gamma) u, \sin(\gamma)), (x,1) \rangle)^{\alpha}_{+} \ d(\mu-\nu)(x) \right| \ d\tau(\nu) \\ &\geq \underline{\mathcal{W}}_{1,1}(\mu,\nu) - 2 C(1,\alpha) \delta^{-\frac{1}{\alpha}} \log \left(\delta \right).
\end{split}    
\end{align}
If we use the Hölder inequality in the left-hand side of \eqref{eq:sup_gamma}, we obtain
\begin{align}
\begin{split} \label{eq:holder}
    &\delta \left( \int_{\mathbb{S}^{d-1}} \sup_{\gamma \in [0,2\pi]} \left|\int (\langle (\cos(\gamma) u, \sin(\gamma)), (x,1) \rangle)^{\alpha}_{+} \ d(\mu-\nu)(x) \right|^3 \ d\tau(\nu) \right)^{1/3} \\ &\geq \underline{\mathcal{W}}_{1,1}(\mu,\nu) - 2 C(1,\alpha) \delta^{-\frac{1}{\alpha}} \log \left(\delta \right).
\end{split}    
\end{align}
By \autoref{lem:tilde_f2_bound}, we have
\begin{align}
\begin{split}
    d^2_{\mathcal{B}_{\tilde{\mathcal{F}}_2}}(\mu,\nu) \geq \frac{1}{2 \pi} \frac{5}{6 \alpha 2^{\alpha/2}} \int_{\mathbb{S}^{d-1}} \sup_{\gamma \in [0,2\pi]} \left|\int \left(\langle (x,1), (\cos(\gamma) \xi_{(d)}, \sin(\gamma)) \rangle \right)^{\alpha}_{+} d(\mu-\nu)(x) \right|^3 \ d\tau_{(d-1)}(\xi_{(d)}).
\end{split}    
\end{align}
Hence, combining this bound with \eqref{eq:holder} we conclude that
\begin{align}
    \delta d^{2/3}_{\tilde{\mathcal{F}}_2}(\mu,\nu) \geq \left(\frac{5}{12 \pi \alpha 2^{\alpha/2}} \right)^{1/3} \left( \underline{\mathcal{W}}_{1,1}(\mu,\nu) - 2 C(1,\alpha) \delta^{-\frac{1}{\alpha}} \log \left(\delta \right) \right).
\end{align}
The upper bound $\pi d^2_{\mathcal{B}_{\tilde{\mathcal{F}}_2}}(\mu,\nu) \leq \underline{\mathcal{W}}_{1,1}(\mu,\nu)$ follows from 
\begin{align}
\begin{split}
    \underline{\mathcal{W}}_{1,1}(\mu,\nu) &= \int_{\mathbb{S}^{d-1}} \left( \sup_{f \in \text{Lip}_1(\R)}\mathbb{E}_{x \sim \mu} [f(\langle u, x \rangle)] - \mathbb{E}_{x \sim \nu} [f(\langle u, x \rangle)] \right) \ d\tau(u)
    \\ &\geq \frac{1}{2} \int_{\mathbb{S}^{d-1}} \left( \sup_{f \in \text{Lip}_1(\R)} \mathbb{E}_{x \sim \mu} [f(\langle u, x \rangle)] - \mathbb{E}_{x \sim \nu} [f(\langle u, x \rangle)] \right)^2 \ d\tau(u)
    \\ &\geq \frac{1}{2} \int_{\mathbb{S}^{d-1}} \left( \sup_{f \in \mathcal{B}_{\mathcal{F}_2(\R)}}\mathbb{E}_{x \sim \mu} [f(\langle u, x \rangle)] - \mathbb{E}_{x \sim \nu} [f(\langle u, x \rangle)] \right)^2 \ d\tau(u) 
    \\ &= \frac{1}{2} \int_{\mathbb{S}^{d-1}} \int_{\mathbb{S}^1} \left( \int (\langle (\langle u, x \rangle,1), \theta \rangle)^{\alpha}_{+} d(\mu-\nu)(x) \right)^2 \ d\tau_{(1)}(\theta) \ d\tau(u)
    \\ &= \frac{1}{2} \int_{\mathbb{S}^{d-1}} \int_{0}^{2\pi} \left( \int \left(\langle (x,1), (\cos(\gamma) u, \sin(\gamma)) \rangle \right)^{\alpha}_{+} d(\mu-\nu)(x) \right)^2 \ d\gamma \ d\tau(u)
    \\ &= \pi d^2_{\mathcal{B}_{\tilde{\mathcal{F}}_2}}(\mu,\nu)
\end{split}
\end{align}
In the first inequality, we used that $|\sup_{f \in \text{Lip}_1(\R)}\mathbb{E}_{x \sim \mu} [f(\langle u, x \rangle)] - \mathbb{E}_{x \sim \nu} [f(\langle u, x \rangle)]| \leq 2$ since $\|x\|_2 \leq 1$ for all $x$ in the support of $\mu$ or $\nu$.
In the second inequality, we used that $\mathcal{B}_{\mathcal{F}_2(\R)} \subseteq \mathcal{B}_{\mathcal{F}_1(\R)} \subseteq \text{Lip}_1(\R)$. The next equality follows from \eqref{eq:d_f2}. The last equality is from \eqref{eq:d_tilde_f2_change}.
\end{proof}

\section{Experimental details} \label{sec:details_exp}
For the the figures, the experiments were run with CPUs from a cluster, using a different 15GB RAM node for each dimension and repetition. The experiments for \autoref{fig:f1_f2_ipm_separation_figure} were the most computationally expensive taking about 40 hours to complete. We need to use a high amount of Monte Carlo samples from the measures to reduce the variance of the estimator, and samples are computationally expensive to obtain because the rejection rate for rejection sampling, which was the method we chose for simplicity, was high. The code would be faster if we had used MCMC methods to obtain the samples, but we are not too concerned about the speed because the only purpose is to plot figures, not to design an algorithm that can be implemented.

\paragraph{Details on \autoref{fig:f1_f2_ipm_separation_figure}.}
To get the theoretical $\mathcal{F}_1$ IPM estimate (which is \textit{still} an estimate, i.e. not a closed form expression), we use \eqref{eq:distances_ratio0}, which states that 
\begin{align}
d_{\mathcal{B}_{\mathcal{F}_1}}(\mu_d,\nu_d) =
\frac{2 \left|\int_{-1}^{1} P_{k,d}(t) \sigma(t) (1-t^2)^{\frac{d-3}{2}} \ dt \right|}{\int_{-1}^{1} |P_{k,d}(t)| (1-t^2)^{\frac{d-3}{2}} dt}.
\end{align}
To approximate this quantity, we observe that it can be expressed as $2\mathbb{E}_{t}[\sigma(t) \text{sign}(P_{k,d}(t))]$ when the distribution of $t$ has a density proportional to $|P_{k,d}(t)| (1-t^2)^{\frac{d-3}{2}}$ restricted to $[-1,1]$. We sample from this density using rejection sampling and obtain the desired estimate as the Monte Carlo estimate of $2\mathbb{E}_{t}[\sigma(t) \text{sign}(P_{k,d}(t))]$. 

The empirical $\mathcal{F}_1$ IPM estimate in the left plot is computed by writing, per \autoref{lem:f1_exp}, $d_{\mathcal{B}_{\mathcal{F}_1}}(\mu_d,\nu_d) = \sup_{\theta \in \mathbb{S}^{d-1}} \left| \int \sigma(\langle x, \theta \rangle) d(\mu_d-\nu_d)(x) \right|$. Since this supremum is attained at $\theta = e_d$ (see the proof of \autoref{lem:d_f1_comp} in \autoref{sec:proof_sep}), we rely on the Monte Carlo estimate $d_{\mathcal{B}_{\mathcal{F}_1}}(\mu_d,\nu_d) \approx \frac{1}{M}|\sum_{i=1}^M \sigma(\langle x_i, e_d \rangle) - \sigma(\langle y_i, e_d \rangle)|$, where $(x_i)_{i=1}^M$ and $(y_i)_{i=1}^M$ are i.i.d. samples from $\mu_d$ and $\nu_d$ respectively. Analogously, the $\mathcal{F}_2$ IPM estimate in the left plot is computed by writing, per \autoref{lem:f2_exp}, $d^2_{\mathcal{B}_{\mathcal{F}_2}}(\mu_d,\nu_d) = \int_{\mathbb{S}^{d-1}} \left( \int \sigma(\langle x, \theta \rangle) d(\mu_d-\nu_d)(x) \right)^2 d\tau(\theta)$. We use Monte Carlo estimates to approximate the integrals over $\mu_d-\nu_d$ and over $\tau$, i.e. 
\begin{align}
d^2_{\mathcal{B}_{\mathcal{F}_2}}(\mu_d,\nu_d) \approx \frac{1}{N} \sum_{j=1}^N \left( \frac{1}{M} \sum_{i=1}^M \sigma(\langle x_i, \theta_j \rangle) - \sigma(\langle y_i, \theta_j \rangle) \right)^2.
\end{align} 
In \autoref{fig:f1_f2_ipm_separation_figure} we used $M = 6000000, N = 10000$ and we obtained the samples $(x_i)_{i=1}^M$ and $(y_i)_{i=1}^M$ using rejection sampling. The curves for the empirical estimates in the left plot are obtained by running the Monte Carlo estimate 10 times; thick lines show the average, and error bars indicate the minimum and maximum values over the 10 repetitions. The empirical ratio in the right plot is obtained by dividing the $\mathcal{F}_1$ IPM estimate over the $\mathcal{F}_2$ IPM estimate, and its error bars are obtained by dividing the minimum value (resp. maximum) for the $\mathcal{F}_1$ IPM over the 10 repetitions by the maximum value (resp. minimum) for the $\mathcal{F}_2$ IPM.

\paragraph{Details on \autoref{fig:f1_f2_sd_separation_figure}.} The $\mathcal{F}_1$ SD estimate in the left plot is computed using that $\text{SD}_{\mathcal{B}_{\mathcal{F}_1^d}}(\mu_d,\nu_d)$ is equal to
\begin{align}
\begin{split}
\gamma_{k,d} \frac{k}{\alpha+1} \lambda_{k,d}^{(\alpha+1)} \sqrt{\sum_{i=1}^{d} \sup_{\theta^{(i)} \in \mathbb{S}^{d-1}} \left( \frac{d+\alpha-2}{k} \hat{\nabla}_i L_{k,d}(\theta^{(i)}) -(\alpha+1-k) L_{k,d}(\theta^{(i)}) \theta^{(i)}_i \right)^2}.
\end{split}
\end{align}
by equations \eqref{eq:sd_f1_development},  \eqref{eq:sup_rewritten} and \eqref{eq:sup_rewritten2}. In \eqref{eq:dim_d_sd} we lower-bound the supremum for $i=d$, which suffices for the lower bound in \autoref{lem:f1_sd_lower}. However, we need a procedure to approximate the suprema for $i = 1,\dots,d$.
By 
the fact that $L_{k,d}(x) = \|x\|^k P_{k,d}(\langle e_d, x \rangle/\|x\|)$ for any $x \in \R^d$ (see the second paragraph of \autoref{sec:sep_f1_f2_ipm}), we have that for all $\theta^{(i)} \in \mathbb{S}^{d-1}$, 
\begin{align} 
    \hat{\nabla}_i L_{k,d}(\theta^{(i)}) &= \hat{\nabla}_i \left( \|\theta^{(i)}\|^k P_{k,d}(\langle e_d, \theta^{(i)} \rangle/\|\theta^{(i)}\|) \right) 
    \\ &= \mathds{1}_{i=d} P'_{k,d}(\langle e_d, \theta^{(i)} \rangle) - \langle e_d, \theta^{(i)} \rangle 
    P'_{k,d}(\langle e_d, \theta^{(i)} \rangle) \theta^{(i)}_i + k P_{k,d}(\langle e_d, \theta^{(i)} \rangle) \theta^{(i)}_i
    \\ &= 
    \frac{k(k+d-2)}{d-1} P_{k-1,d+2}(\langle e_d, \theta^{(i)} \rangle) (\mathds{1}_{i=d} -\langle e_d, \theta^{(i)} \rangle \theta^{(i)}_i) + k P_{k,d}(\langle e_d, \theta^{(i)} \rangle) \theta^{(i)}_i
\end{align}
In the last equality we have used \eqref{eq:legendre_derivatives}. Thus, for $i \neq d$,
\begin{align} \label{eq:euclidean_grad_i}
    &\frac{d+\alpha-2}{k} \hat{\nabla}_i L_{k,d}(\theta^{(i)}) -(\alpha+1-k) L_{k,d}(\theta^{(i)}) \theta^{(i)}_i \\ &= - \frac{(d+\alpha-2)(k+d-2)}{d-1}\langle e_d, \theta^{(i)} \rangle 
    P_{k-1,d+2}(\langle e_d, \theta^{(i)} \rangle) \theta^{(i)}_i + (d+k-3) P_{k,d}(\langle e_d, \theta^{(i)} \rangle) \theta^{(i)}_i
\end{align}
We want to find $\theta^{(i)}$ that maximizes the absolute value of \eqref{eq:euclidean_grad_i} within $\mathbb{S}^{d-1}$, which via the change of variables $t = \langle e_d, \theta^{(i)} \rangle$ is equivalent to minimizing the one-dimensional function
\begin{align} \label{eq:1-d_problem}
    \left| - \frac{(d+\alpha-2)(k+d-2)}{d-1} t (1-t^2)^{1/2}
    P_{k-1,d+2}(t) + (d+k-3) (1-t^2)^{1/2} P_{k,d}(t) \right|
\end{align}
over $[-1,1]$. Here, we have used that the absolute value of \eqref{eq:euclidean_grad_i} is maximized when $\theta^{(i)} = \theta^{(i)}_i + \langle e_d, \theta^{(i)} \rangle e_d$, which implies that $\theta^{(i)}_i = \pm (1-t^2)^{1/2}$.
We can optimize \eqref{eq:1-d_problem} over $[-1,1]$ via brute force, since it is a one dimensional problem. On the other hand, when $i=d$ we have
\begin{align} \label{eq:euclidean_grad_ii}
    &\frac{d+\alpha-2}{k} \hat{\nabla}_i L_{k,d}(\theta^{(i)}) -(\alpha+1-k) L_{k,d}(\theta^{(i)}) \theta^{(i)}_i \\ &= \frac{(d+\alpha-2)(k+d-2)}{d-1} P_{k-1,d+2}(\langle e_d, \theta^{(i)} \rangle) (1-\langle e_d, \theta^{(i)} \rangle \theta^{(i)}_i) + (d+k-3) P_{k,d}(\langle e_d, \theta^{(i)} \rangle) \theta^{(i)}_i
\end{align}
We again the change of variables $t = \langle e_d, \theta^{(i)} \rangle$, which in this case implies that $t = \theta^{(i)}_i$. Thus, the problem to be solved for $i=d$ is
\begin{align} \label{eq:1-d_problem2}
    \left|\frac{(d+\alpha-2)(k+d-2)}{d-1}
    P_{k-1,d+2}(t) (1-t^2) + (d+k-3) P_{k,d}(t) t \right|.
\end{align}
The theoretical lower bound on the $\mathcal{F}_1$ SD is obtained directly by evaluating the right-hand side of \eqref{eq:f1_sd_lower}.

The $\mathcal{F}_2$ SD estimate is obtained as a Monte-Carlo estimate of the right-hand side of \eqref{eq:rh_1}. Namely, if $(\theta_j)_{j=1}^N$ and $(x_l)_{l=1}^M$ are uniform i.i.d. samples over $\mathbb{S}^{d-1}$, 
\begin{align}
    d^2_{\mathcal{B}_{\mathcal{F}_2}}(\mu_d,\nu_d) \approx \gamma_{k,d}^2 \sum_{i=1}^{d} \frac{1}{N} \sum_{j=1}^N \left(\frac{1}{M} \sum_{l=1}^{M} \nabla_i L_{k,d}(x_l) (\langle x_l, \theta_j \rangle)_{+}^{\alpha} \right)^2.
\end{align}

\paragraph{Details on \autoref{fig:f1_f2_wasserstein_figure}.} Denoting by $\mu_d$ the standard $d$-variate Gaussian and by $\nu_d$ the $d$-variate Gaussian with unit variance in all directions except for $e_d$ with variance 0.1,  we have taken $M$ samples $(x_i)_{i=1}^M$ of $\mu_d$ and $M$ samples $(y_i)_{i=1}^M$ of $\mu_d$. We used the same estimate for $d_{\mathcal{B}_{\mathcal{F}_1}}(\mu_d, \nu_d)$ and $d_{\mathcal{B}_{\mathcal{F}_2}}(\mu_d, \nu_d)$ as in \autoref{fig:f1_f2_ipm_separation_figure}, although in this case with bias term and with $M = 100000$ and $N=10000$. To obtain an estimate $d_{\mathcal{B}_{\tilde{\mathcal{F}}_2}}(\mu_d, \nu_d)$, we sample $N$ uniform samples $(\theta_i)_{i=1}^N$ from $\mathbb{S}^{d-1}$ and $N$ uniform samples $(t_i)$ from $[-1,1]$ and we compute
\begin{align}
d^2_{\mathcal{B}_{\mathcal{F}_2}}(\mu_d,\nu_d) \approx \frac{1}{N} \sum_{j=1}^N \left( \frac{1}{M} \sum_{i=1}^M \sigma(\langle (x_i,1), (\sqrt{1-t_j^2} \theta_j, t_j) \rangle) - \sigma(\langle (y_i,1), (\sqrt{1-t_j^2} \theta_j, t_j) \rangle) \right)^2.
\end{align}
Let $\mathcal{W}_{1}((\theta)_{\#} \mu, (\theta)_{\#}\nu)$ be the one-dimensional Wasserstein distance between the projections of $\mu$ and $\nu_d$ to the one dimensional subspace spanned by $\theta \in \mathbb{S}^{d-1}$. Let $\hat{\mu}_d = \frac{1}{M} \sum_{i=1}^M \delta_{x_i}$ and $\hat{\nu}_d = \frac{1}{M} \sum_{i=1}^M \delta_{y_i}$.
To estimate the max-sliced Wasserstein $\overline{\mathcal{W}}_{1,1}(\hat{\mu}_d, \hat{\nu}_d)$ we compute $\overline{\mathcal{W}}_{1,1}(\mu_d, \nu_d) \approx \mathcal{W}_{1}((e_d)_{\#} \hat{\mu}_d, (e_d)_{\#} \hat{\nu}_d)$ to $e_d$, because we know that in theory $e_d$ is the direction of maximal discrepancy. The one-dimensional Wasserstein distance can be computed quickly.

To estimate the sliced Wasserstein distance $\underline{\mathcal{W}}_{1,1}(\mu_d, \nu_d)$ we sample $N$ uniform samples $(\theta_i)_{i=1}^N$ from $\mathbb{S}^{d-1}$ and we compute
\begin{align}
    \underline{\mathcal{W}}_{1,1}(\mu_d, \nu_d) \approx \frac{1}{N} \sum_{j=1}^N \mathcal{W}_{1}((\theta_j)_{\#} \hat{\mu}_d, (\theta_j)_{\#} \hat{\nu}_d).
\end{align}